\documentclass[hidelinks,onefignum,onetabnum]{siamart220329}

\usepackage{lipsum}
\usepackage{amsfonts}
\usepackage{graphicx}
\usepackage{epstopdf}
\usepackage{algorithmic}
\ifpdf
  \DeclareGraphicsExtensions{.eps,.pdf,.png,.jpg}
\else
  \DeclareGraphicsExtensions{.eps}
\fi

\newsiamremark{remark}{Remark}
\newsiamremark{hypothesis}{Hypothesis}
\crefname{hypothesis}{Hypothesis}{Hypotheses}
\newsiamthm{claim}{Claim}

\headers{Efficient Identification under Saturated Observation}{Lantian Zhang, Lei Guo}

\title{Asymptotically Efficient Adaptive Identification under Saturated Output Observation \thanks{This version: May 31, 2024.
\funding{This work was supported by the National Natural Science Foundation of China under Grant No. 12288201.}}}

\author{Lantian Zhang\thanks{
		(\email{zhanglantian@amss.ac.cn}, \email{Lguo@amss.ac.cn}).}
	\and Lei Guo\footnotemark[2]}

\author{Lantian Zhang\thanks{Department of Mathematics, KTH Royal Institute of Technology, Stockholm, Sweden.
  (\email{lantian@kth.se}).}
\and Lei Guo 
\thanks{Key Laboratory of Systems and Control, Academy of Mathematics and Systems Science, Chinese Academy of Sciences, Beijing, P.R. China.
  (\email{Lguo@amss.ac.cn}).}
}

\usepackage{amsopn}
\DeclareMathOperator{\diag}{diag}

\usepackage{array} 
\usepackage[caption=false]{subfig} 
\newcolumntype{R}{>{$}r<{$}} %
\newcolumntype{V}[1]{>{[\;}*{#1}{R@{\;\;}}R<{\;]}} %
\usepackage{amssymb}
 \usepackage{graphicx,epstopdf} 
 \usepackage[caption=false]{subfig} 
\newtheorem{assumption}{Assumption}[section]
 
\ifpdf
\hypersetup{
  pdftitle={Asymptotically Efficient Algorithm},
  pdfauthor={Lantian Zhang, Lei Guo}
}
\fi

\begin{document}

\maketitle

\begin{abstract}
As saturated output observations are ubiquitous in practice, identifying stochastic systems with such nonlinear observations is a fundamental problem across various fields. This paper investigates the asymptotically efficient identification problem for stochastic dynamical systems with saturated output observations. In contrast to most of the existing results, our results do not need the commonly used but stringent conditions such as periodic or independent assumptions on the system signals, and thus do not exclude applications to stochastic feedback systems. To be specific, we introduce a new adaptive Newton-type algorithm on the negative log-likelihood of the partially observed samples using a two-step design technique. Under some general excitation data conditions, we show that the parameter estimate is strongly consistent and asymptotically normal by employing the stochastic Lyapunov function method and limit theories for martingales. Furthermore, we show that the mean square error of the estimates can achieve the Cram\'er-Rao bound asymptotically without
resorting to i.i.d data assumptions. This indicates that the performance of the proposed algorithm is the best possible that one can expect in general.  A numerical example is provided to illustrate the superiority of our new adaptive algorithm over the existing related ones in the literature. 	
\end{abstract}

\begin{keywords}
system identification, stochastic systems, saturated output observations, adaptive algorithm, Cram\'er-Rao bound.
\end{keywords}

\begin{MSCcodes}
93E12, 68T05, 93E03, 93C10, 68W40    
\end{MSCcodes}
{\small 
\section{Introduction}
\subsection{Background}
Stochastic systems with saturated output observations are an important class of nonlinear dynamical systems, as they effectively model the nonlinear phenomena caused by observation saturation—a widespread phenomenon in fields such as engineering \cite{sj2004, hf2009, liu2001}, biology \cite{li2003, Scherm2004}, social sciences \cite{mj2020, cm2013, ZZtsqn}. Besides, this kind of observation function can encompass several important cases, for instance, the binary-valued observation function applied in classification tasks and quantized communication, as well as the ReLU-type observation function utilized in machine learning.

Due to their importance, extensive research has been devoted to the stabilization and optimization of systems with output sensor saturation nonlinearity (see, for example, \cite{Kreisselmeier, cao2003, Xie2023}). However, in many real-world scenarios, the models of dynamical systems are often unknown, thereby requiring model learning from observations. This paper considers the identification problem with saturated output observations, which has attracted much attention and has been applied in a wide range of application domains, including air-fuel ratio control in ICEs  \cite{gagliardi2021}, judicial sentencing  \cite{ZZtsqn, judical},  binary classification tasks \cite{mei}, international trade  \cite{cm2013}, restoration of images  \cite{liu2001}, ReLU neuron's learning  \cite{yehudai}, among others. 

Over the past several decades, in the literature related to the identification problem of stochastic systems with saturated output observations, numerous methods have been proposed to estimate the unknown parameters. These include the empirical measure method (\cite{ZG2003}), the Leave-out Sign-dominant Correlation Regions algorithm (\cite{weyer2009finite}), the EM algorithm (\cite{godoy2011identification}), stochastic approximation-type algorithms (\cite{jafari:2012},\cite{guo:2013}), stochastic Newton-type algorithms (\cite{bercu}, \cite{ZZtsqn}), and so on. All these works and methods have made extensive progress in the convergence analysis of algorithms. However, knowledge about the corresponding Cram\'er-Rao (C-R) bound and the efficiency of parameter estimates remains limited. In the estimation theory, when comparing two unbiased estimators, the one with a smaller variance is considered to be better. It is widely recognized that there exists an irreducible lower bound for unbiased parameter estimators, known as the C-R bound. If the C-R bound is achieved, the corresponding estimator is referred to as ``asymptotically efficient" (see, e.g., \cite{shao2003}). 

For related works on the C-R bounds under nonlinear observations, we note that the C-R bounds for stochastic  Wiener systems are derived in \cite{ne} and \cite{bo}, where the observation sensors are general nonlinear functions. Additionally, the C-R bound for quantized measurement systems is studied in \cite{gusta}, providing a comprehensive analysis of the impact of dithering noise on the C-R bound. However, the construction of parameter estimation algorithms and theoretical analysis of efficiency with saturated output observations have encountered technical difficulties and still lack satisfactory theoretical results. In fact, almost all the existing studies require independence or periodicity conditions on regressors. For example,
the asymptotical efficiency properties of the empirical measure method are explored in \cite{wang2007, guo2015, zhao2023}, where the input signals are assumed to be deterministic and periodic. Besides, a stochastic approximation-type adaptive estimator with adaptive binary observations is considered in \cite{you2015}, where the C-R bound is asymptotically achieved under i.i.d. input signals. Moreover, the maximum likelihood (ML) estimates are considered in \cite{ly1992}, and the asymptotic efficiency of the algorithm is established under deterministic signals satisfying a strong persistent excitation(PE) condition. Furthermore, adaptive second-order algorithms were considered in \cite{wangying2023},  and the asymptotic efficiency is proven under deterministic signals and a strong noise condition. Though the independence or periodic conditions on data are convenient for theoretical investigation, they are hard to satisfy or verify in many important situations, e.g. feedback signals in adaptive control \cite{1412020, guo1995}, sequential decision-making contexts \cite{doy}. The asymptotic efficiency of estimation algorithms under saturated output observations and stochastic non-i.i.d. data conditions still poses a challenging open problem.

Fortunately, there are many analytical methods for identifying linear or nonlinear stochastic systems under non-i.i.d. data conditions that can be drawn on to solve this problem. Indeed, motivated by the need to establish a rigorous theory for the well-known LS-based self-tuning regulators proposed by \cite{astrom1973} in stochastic adaptive control, the analysis of adaptive algorithms with stochastic feedback signals had witnessed a great deal of progress in the literature.  Among the theoretical analyses of adaptive algorithms, we mention that the ordinary differential equation (ODE) method was established in \cite{ljung} for a broad class of recursive algorithms. Here, the conditions for regressors are replaced by stability conditions for the corresponding ODE. Moreover, the stochastic Lyapunov function method was considered in  \cite{moore:1978, lai:1982} for the classical least squares (LS) algorithm of linear stochastic regression models. The strong consistency of LS was successfully established under the weakest possible non-PE condition in  \cite{lai:1982}. Recently, in our previous works \cite{ZZ2022} and \cite{ZZtsqn}, we established the strong consistency of adaptive algorithms for stochastic systems with saturated output observations under general non-PE conditions by following the stochastic Lyapunov function method. Although the data conditions required for the strong consistency of our proposed algorithm have been greatly weakened, how to design an adaptive algorithm with guaranteed optimality still remains an open problem. This paper will solve this problem by developing an adaptive algorithm that achieves the C-R bound under saturated output observations and stochastic non-i.i.d. data conditions.

\subsection{Contributions}

To achieve the C-R bound without relying on i.i.d conditions for regressors, we introduce an asymptotically efficient adaptive algorithm inspired by the design idea of the recursive least squares algorithm and employ the stochastic Lyapunov function method to analyze its properties.  The main contributions of this paper can be summarized as follows: Firstly, we propose a new two-step online recursive algorithm, where the first step focuses on achieving convergence and the second step is dedicated to improving the algorithm performance. Secondly, we demonstrate that the algorithm is strongly consistent under a quite general excitation condition on the data. Our excitation data condition coincides with the weakest possible excitation condition known for the classical least square algorithm in stochastic linear regression models. Thirdly, we show that the estimation errors are asymptotically normal and the covariance of the estimates can asymptotically approach the well-known C-R bound. It is worth noting that all of our theories do not need the independence or periodicity of the system signals. To the best of the authors’ knowledge, this paper appears to be the first to achieve the C-R bound of the identification algorithm for stochastic systems with saturated output observations under non-i.i.d data conditions.

The rest of the paper is organized as follows: Section \ref{sec2} introduces the notations and assumptions used in later sections. In Section \ref{sec3}, we propose the asymptotic efficient online algorithm and present the main theorems of the paper. Section \ref{sec4} gives the proofs of the main results. Section \ref{sec5} offers numerical simulations to illustrate the theoretical results and compare the performance with other existing algorithms. Finally, Section \ref{sec6} provides concluding remarks.

\section{Problem Setup}\label{sec2}

Throughout the sequel, we need the following notations:
By $\|\cdot\|$, we denote the Euclidean norm of vectors or matrices. By $\|\cdot\|_{p}$ we mean the $L_{p}-$norm defined by $\|X\|_{p}=\left\{\mathbb{E}\|X\|^{p}\right\}^{\frac{1}{p}}$, where $p\in \mathbb{R}$. The  maximum and minimum eigenvalues of a symmetric matrix $M$ are denoted by $\lambda_{max}\left\{M\right\}$ and $\lambda_{min}\left\{M\right\}$  respectively. Besides, let $tr(M)$ denote the trace of the matrix $M$, and by $|M|$  we mean the determinant of the matrix $M$. Moreover, $\left\{\mathcal{F}_{k},k\geq 0\right\}$ is a non-decreasing sequence of $\sigma -$algebras, and the corresponding  conditional mathematical expectation operator is $\mathbb{E}[\cdot \mid \mathcal{F}_{k}]$, we may employ the abbreviation $\mathbb{E}_{k}\left[\cdot\right]$ to $\mathbb{E}[\cdot \mid \mathcal{F}_{k}]$ in the sequel.

Consider the following discrete-time stochastic system: 
\begin{equation}\label{eq1}
	\begin{aligned}
		y_{k+1}&=\phi_{k}^{\top}\theta+v_{k+1},\;\;\;	k=0,1,\cdots\\
		s_{k+1}&=S_{k}(y_{k+1}),
	\end{aligned}
\end{equation}	
where  $\theta\in \mathbb{R}^{m}$ is an unknown parameter vector to be estimated; $y_{k+1}\in \mathbb{R}$, $s_{k+1}\in \mathbb{R}$, $\phi_{k}\in \mathbb{R}^{m}$, $v_{k+1}\in \mathbb{R}$ represent the system output, system output observation, system signal regressor, and random noise process, respectively. Besides, $S_{k}(\cdot)$ is a saturation function defined by 
\begin{equation}\label{eq2}
	S_{k}(x)=\left\{
	\begin{array}{rcl}
		&L_{k}             & {x     <     l_{k}}\\
		&x         & {l_{k} \leq x \leq u_{k}}\\
		&U_{k}      & {x > u_{k}}
	\end{array} \right.,\;\;\;k=0,1,\cdots.
\end{equation}
At each time instant, the noise-corrupted output can be precisely observed only if its value falls within the specified interval $[l_{k}, u_{k}]$. If the output value exceeds this range, the output becomes saturated, resulting in imprecise information represented by $L_{k}$ or $U_{k}$. 

The problem considered here is to develop adaptive algorithms to estimate the unknown parameter $\theta$ using nonlinear observations $\{\phi_{k}, s_{k+1}\}$. Let $(s_{1},\cdots,s_{n})$ be the saturated output observations available at time instant $n$ with the joint probability density function $p_{\theta}\in \mathcal{P}=\{p_{\theta}: \theta \in \Theta\}$, where $\Theta$ is an open set in $\mathbb{R}^{m}$. Assume that the regularity conditions of the Fr\'echet-Cram\'er-Rao (FCR) inequality (cf.,\cite{shao2003}) are satisfied by the family $\mathcal{P}$, then according to this inequality, for each $n\geq 1$, the efficient unbiased estimator $\hat{\theta}_{n}$ for the parameter $\theta$ satisfies  
\begin{equation}\label{2.1}
	\Delta_{n}^{\frac{1}{2}}Cov(\hat{\theta}_{n})\Delta_{n}^{\frac{1}{2}}\geq I_{m},
\end{equation}
where $I_{m}$ is the identity matrix of order $m$, $Cov(\hat{\theta}_{n})$ is the covariance of the estimate $\hat{\theta}_{n}$, and $\Delta_{n}$ is the Fisher information matrix with the measurements $(s_{1},\cdots, s_{n}), $ defined by 
\begin{equation}\label{2.2}
	\Delta_{n}=\mathbb{E}\left[\frac{\partial \log p_{\theta}(s_{1},\cdots,s_{n})}{\partial \theta}\frac{\partial \log p_{\theta}(s_{1},\cdots,s_{n})}{\partial \theta}^{\top}\right].
\end{equation}

The objective of this paper is to introduce an asymptotically unbiased adaptive algorithm for non-i.i.d. data, that can attain the C-R bound in the asymptotic sense, i.e., the equality of $(\ref{2.1})$ is satisfied in the asymptotic sense: 
\begin{equation}\nonumber
	\lim\limits_{n\rightarrow \infty}\Delta_{n}^{\frac{1}{2}}Cov(\hat{\theta}_{n})\Delta_{n}^{\frac{1}{2}}= I_{m}.
\end{equation}

For the theoretical analysis, we introduce the  following basic assumptions:
 \begin{assumption}\label{assum2}
		 For each $k\geq 1$, the norm of the system regressor $\phi_{k}\in \mathcal{F}_{k}=\sigma\{s_{1},\cdots,s_{k}\}$ is bounded by a constant. Besides, the parameter $\theta$ is an interior point of a known convex compact set $\Theta\subseteq \mathbb{R}^{m}$, and denote $D=\sup\limits_{x\in \Theta}\|x\|$.
 \end{assumption}

 	 \begin{assumption}\label{assum3}
        The thresholds $\left\{l_{k}, u_{k}, L_{k}, U_{k}, k\geq 0\right\}$ are known deterministic \\
        sequences satisfying for any $k\geq 0$,
		\begin{equation}\label{5}
				L_{k}\leq l_{k}\leq u_{k} \leq U_{k},	\;\;\;L_{k}<U_{k},	
		\end{equation}
		and
		\begin{equation}\label{6}
			\sup_{k\geq 0}\{l_{k}\} <\infty, 	\;\;\inf_{k\geq 0}\{u_{k}\}> -\infty.
		\end{equation}	
 \end{assumption}

 \begin{assumption}\label{assum4}
	For each $k\geq 1$, the system noise $v_{k}$ is $\mathcal{F}_{k}-$measurable and follows a conditional normal distribution $N(0,\sigma^{2})$ given $\mathcal{F}_{k-1}$.
\end{assumption}

\begin{remark}
The boundedness condition on regressors and the conditional normality condition on noise in Assumption \ref{assum2} and Assumption \ref{assum4} ensure the required growth rate of the minimum eigenvalue of the Fisher information matrix $\Delta_{n}$, specifically, ensuring that the scalar scalar $\{\lambda_{n}, n\geq 0\}$ in $(\ref{2.9})$ has a positive lower bound. For extension, the boundedness condition on regressors can be relaxed to allow the norms of the regressor vector to grow at a certain rate, by permitting $\lambda_{k}$ to converge to $0$ at a certain rate, as analyzed in \cite{ZZ2024}. Moreover, the computation of the likelihood function requires the knowledge of the noise variance $\sigma^{2}$, which will be used in the algorithm design. If the variance $\sigma$ is unknown,  it may be estimated either directly by using the observation data (see the sentencing problems studied in \cite{judical}), or estimated jointly with system parameters by increasing the dimension of the regressors by 1, see, e.g., \cite{ke2024}. 
\end{remark}

\begin{remark}
	Assumption \ref{assum3} can be satisfied by various classical models.  
	Below we give three concrete examples. Firstly, if $L_{k}=l_{k}>-\infty$ and $U_{k}=u_{k}<\infty$, then the model turns to the interval-censored regression model. 
	Secondly, if $l_{k}=u_{k}=0,$ $L_{k}=0$, and $U_{k}=1$, then the model turns to the well-known McCulloch-Pitts model widely used in binary classification 
	or the binary-valued observation model considered in signal processing. 	Thirdly, if $L_{k}=l_{k}=0,$ and $u_{k}=U_{k}=\infty$, then the model becomes the one-layer ReLu network model. 
	Moreover, the inequality $L_{k}<U_{k}$ implies that the observed values of the system are not constant, which is a natural condition for system identifiability.
\end{remark}

\section{Main results}\label{sec3}
\subsection{Fisher information}
To proceed with further discussions, we first investigate the Fisher information in the current nonlinear case.

Under Assumptions $\ref{assum2}$-$\ref{assum4}$ and assuming that $\phi_{0}$ is deterministic, for each $n\geq 1$, the joint density of system output observation sequence at $(s_{1},\cdots,s_{n})$ can be written as $p(s_{1},\cdots,s_{n})
=p(s_{1}\mid \phi_{0})\cdots  p(s_{n}\mid \phi_{0}, \cdots, s_{n})
$ by the Bayes rule.
Let $F(\cdot)$ and $f(\cdot)$ be the conditional distribution function and density function of the noise $v_{k+1}$ given $\mathcal{F}_{k}$, respectively. Then the likelihood function is
\begin{equation}\label{3.2}
\setlength\abovedisplayskip{6pt}
\setlength\belowdisplayskip{6pt}
	\begin{aligned}
		\mathcal{L}_{n}(\theta)
		=&\prod_{k=0}^{n-1}\ell_{k}(\theta)\\
		=&\prod_{k=0}^{n-1}\left[1-F(u_{k}-\phi_{k}^{\top}\theta)\right]^{\bar{\delta}_{k}}\prod_{k=0}^{n-1}\left[F(l_{k}-\phi_{k}^{\top}\theta)\right]^{\delta_{k}}
\prod_{k=0}^{n-1}\left[f(y_{k+1}-\phi_{k}^{\top}\theta)\right]^{1-\delta_{k}-\bar{\delta}_{k}},
	\end{aligned}
\end{equation}
where $\delta_{k}=I\left(y_{k+1}<l_{k}\right),\;\;\;\bar{\delta}_{k}=I\left(y_{k+1}>u_{k}\right).$
and $I\left(\cdot\right)$ is the indicator function. From $(\ref{3.2})$, one can check that the derivative of the logarithm of $\mathcal{L}_{n}(\theta)$ is
\begin{equation}\label{3.3}
	\begin{aligned}	
		\frac{\partial \log\mathcal{L}_{n}(\theta)}{\partial \theta}=&\sum_{k=0}^{n-1}\phi_{k}\left\{\frac{-f(l_{k}-\phi_{k}^{\top}\theta)}{F(l_{k}-\phi_{k}^{\top}\theta)}\delta_{k}+\frac{y_{k+1}-\phi_{k}^{\top}\theta}{\sigma^{2}}(1-\delta_{k}-\bar{\delta}_{k})+\frac{f(u_{k}-\phi_{k}^{\top}\theta)}{1-F(u_{k}-\phi_{k}^{\top}\theta)}\bar{\delta}_{k}\right\}\\
		=&\sum_{k=0}^{n-1}\phi_{k}\mathcal{H}_{k}(\theta),
	\end{aligned}	
\end{equation}
where the function $\mathcal{H}_{k}(\cdot)$ is defined, for all $k\geq 0$ and $x \in \mathbb{R}^{m}$, by 
\begin{equation}\label{344}
\begin{aligned}
		\mathcal{H}_{k}(x)
		=&\frac{-f(l_{k}-\phi_{k}^{\top}x)}{F(l_{k}-\phi_{k}^{\top}x)}\delta_{k}+\frac{y_{k+1}-\phi_{k}^{\top}x}{\sigma^{2}}(1-\delta_{k}-\bar{\delta}_{k})+\frac{f(u_{k}-\phi_{k}^{\top}x)}{1-F(u_{k}-\phi_{k}^{\top}x)}\bar{\delta}_{k}.
\end{aligned}
\end{equation}
Moreover, let the function $G_{k}(\cdot)$ be defined, for all $k\geq 0$ and $x, y\in \mathbb{R}$, by
\begin{equation}\label{3.6}
	\begin{aligned}
		G_{k}(y, x)
		=&\frac{-f(l_{k}-y)}{F(l_{k}-y)}F(l_{k}-x)+\frac{1}{\sigma^{2}}\int_{l_{k}}^{u_{k}}(t-y)dF(t-x)\\
		&+\frac{f(u_{k}-y)}{1-F(u_{k}-y)}\left[1-F(u_{k}-x)\right],
	\end{aligned}
\end{equation}	
then we have
\begin{equation}\label{3.5}
	\begin{aligned}	
		\mathbb{E}\left[\mathcal{H}_{k}(x)\mid \mathcal{F}_{k}\right]
		=G_{k}\left(\phi_{k}^{\top}x,\phi_{k}^{\top}\theta\right),
	\end{aligned}	
\end{equation} 
and the Fisher information matrix $(\ref{2.2})$ is given by the following lemma:
\begin{lemma}\label{lem23}
	For the nonlinear system $(\ref{eq1})$-$(\ref{eq2})$, if Assumptions \ref{assum2}-\ref{assum4} hold, then for each $n\geq 1$, the Fisher information matrix is given by 
	\begin{equation}\label{2100}				\Delta_{n}=\mathbb{E}\left[\sum_{k=0}^{n-1}\lambda_{k}\phi_{k}\phi_{k}^{\top}\right],
	\end{equation}	
	where 
	\begin{equation}\label{2099}
	\lambda_{k}=\frac{f^{2}(l_{k}-\phi_{k}^{\top}\theta)}{F(l_{k}-\phi_{k}^{\top}\theta)}+\int_{l_{k}}^{u_{k}}\frac{[f'(t-\phi_{k}^{\top}\theta)]^{2}}{f(t-\phi_{k}^{\top}\theta)}dt+\frac{f^{2}(u_{k}-\phi_{k}^{\top}\theta)}{1-F(u_{k}-\phi_{k}^{\top}\theta)}.
		\end{equation}
\end{lemma}	
From Assumption $\ref{assum4}$, one can observe that the conditional density function of the noise $v_{k+1}$ given $\mathcal{F}_{k}$ satisfies the following property:
$$f'(x)=-\frac{x}{\sigma^{2}}f(x),\;\; x\in \mathbb{R}.$$
Then, by (\ref{2099}) and the definition of the function $G_{k}(\cdot, \cdot)$, it can be verified that
\begin{equation}\label{2.9}
	\begin{aligned}
	\lambda_{k}=\frac{\partial G_{k}(\phi_{k}^{\top}\theta, x)}{\partial x}|_{x=\phi_{k}^{\top}\theta},
		\end{aligned}
	\end{equation}
which will be used in the construction of the adaptive asymptotically efficient algorithm in the next subsection.
\subsection{Adaptive asymptotically efficient algorithm} 
 
To construct the parameter estimation algorithm, it follows from $(\ref{3.2})$, $(\ref{3.3})$ and $(\ref{3.5})$ that the unknown true parameter $\theta$
satisfies:
 \begin{equation}
\sum_{k=0}^{n-1}\frac{\partial \mathbb{E}\left[\log \ell_{k}(\theta)\mid \mathcal{F}_{k}\right]}{\partial \theta}=\sum_{k=0}^{n-1}\phi_{k}\mathbb{E}\left[\mathcal{H}_{k}(\theta)\mid \mathcal{F}_{k}\right]=\sum_{k=0}^{n-1}\phi_{k}G_{k}(\phi_{k}^{\top}\theta, \phi_{k}^{\top}
\theta)=0.
 \end{equation}	
Therefore, the unknown parameter $\theta$ is the minimizer of the following objective function:
\begin{equation}\label{399}
\theta=\mathop{\arg\min}\limits_{x\in \mathbb{R}^{m}} \sum_{k=0}^{n-1}\mathbb{E}\left[\log l_{k}(x)\mid \mathcal{F}_{k}\right].
 \end{equation}
It is well-known that the offline ML estimates can be obtained by solving the optimization problem $(\ref{399})$. The asymptotic efficiency of offline ML estimates has been extensively studied in the literature under i.i.d. or deterministic regressor conditions (see, e.g., \cite{ly1992}). In this paper, to relax these stringent data conditions, we will propose an adaptive Newton-type algorithm based on the likelihood objective function (\ref{399}).

From Lemma \ref{lem23},  a natural way to obtain recursive estimates with the stochastic Newton (SN) method for all $k\geq 0$ is as follows:
\begin{equation}\label{sn}
\begin{aligned}
\hat{\theta}_{k+1}&=\hat{\theta}_{k}+P_{k+1}\phi_{k}\mathcal{H}_{k}(\hat{\theta}_{k})\\
P_{k+1}&=P_{k}-\hat{\lambda}_{k}a_{k}P_{k}\phi_{k}\phi_{k}^{\top}P_{k},\;\; a_{k}=\frac{1}{1+\hat{\lambda}_{k}\phi_{k}^{\top}P_{k}\phi_{k}},\\
 \hat{\lambda}_{k}&=\frac{\partial G_{k}(\phi_{k}^{\top}\hat{\theta}_{k}, x)}{\partial x}|_{x=\phi_{k}^{\top}\hat{\theta}_{k}},
\end{aligned}
\end{equation}
where the initial values $\hat{\theta}_{0}$ and $P_{0}$ can be arbitrarily chosen in $\Theta$ and with $P_{0}>0$, respectively. 
Using the Sherman–Morrison formula \cite{hager},  it follows that
\begin{equation}
P_{n+1}=\left(\sum\limits_{k=0}^{n}\hat{\lambda}_{k}\phi_{k}\phi_{k}^{\top}\right)^{-1}, 
\end{equation}
which serves as an estimate of the inverse of the Fisher information matrix.
Unfortunately, we cannot prove that the covariance of the estimation error generated by the SN algorithm $(\ref{sn})$ achieves the C-R bound. For this reason, we adopt the design method of the two-step adaptive Newton identification algorithm for nonlinear output observations proposed in \cite{ZZtsqn}, modify the estimates of $\lambda_{k}$, and prove that, when 
$n$ is sufficiently large, the adaptation gain matrix of the proposed algorithm can approach the inverse of the Fisher information matrix.

Moreover, to ensure the boundedness of the estimates generated by our algorithm, we introduce the following projection operator:
\begin{definition}\label{def2}
	For the  convex compact set $\Theta$ in Assumption  $\ref{assum2}$,  the projection operator $\Pi_{A}(x)(\cdot)$ is defined as
	\begin{equation}\label{8}
		\Pi_{A}\{x\}=\mathop{\arg\min}_{y \in \Theta}\|x-y\|_{A}, \quad \forall x \in 	\mathbb{R}^{m},
	\end{equation}
	where the norm $\|\cdot\|_{A}$ associated with  a positive definite matrix $A$ is defined by $\|x\|_{A}=\sqrt{x^{\top}Ax}$.
\end{definition}

Based on the above ideas, our asymptotically efficient algorithm is constructed as Algorithm \ref{alg2}.
\begin{algorithm}[htb]
	\caption{Asymptotically efficient algorithm} 
	\label{alg2} 
	
	\textbf{Step 1.} Recursively calculate the preliminary $\bar{\theta}_{k+1}$:
		\begin{eqnarray}
			\bar{\theta}_{k+1}&=&\Pi_{\bar{P}_{k+1}^{-1}}\{\bar{\theta}_{k}+\bar{\beta}_{k}\bar{P}_{k+1}\phi_{k}\mathcal{H}_{k}(\bar{\theta}_{k})\}, \label{be1}\\
			\bar{P}_{k+1}&=&\bar{P}_{k}-\bar{a}_{k}\bar{\beta}^{2}_{k}\bar{P}_{k}\phi_{k}\phi_{k}^{\top}\bar{P}_{k},\;\;\bar{a}_{k}=\frac{1}{1+\bar{\beta}_{k}^{2}\phi_{k}^{\top}\bar{P}_{k}\phi_{k}},\\
			\bar{\beta}_{k}&=&\min\left\{\underline{g}_{k}, \frac{1}{2\overline{g}_{k}\phi_{k}^{\top}\bar{P}_{k}\phi_{k}+1}\right\}. \label{bebar}
		\end{eqnarray}
	where  $\underline{g}_{k}=\min\limits_{
		x\in \Theta} \frac{\partial G_{k}(\phi_{k}^{\top}\bar{\theta}_{k}, \phi_{k}^{\top}x)}{\partial x}$, $\overline{g}_{k}=\max\limits_{x\in \Theta} \frac{\partial G_{k}(\phi_{k}^{\top}\bar{\theta}_{k}, \phi_{k}^{\top}x)}{\partial x}$, and the function $G_{k}(\cdot, \cdot)$ is defined in $(\ref{3.6})$; the function $\mathcal{H}_{k}(\cdot)$ is defined in $(\ref{344})$; the projection operator $\Pi_{\bar{P}_{k+1}^{-1}}\{\cdot\}$ is defined by 
$(\ref{8})$; the initial values $\bar{\theta}_{0}$ and $\bar{P}_{0}$ can be chosen arbitrarily in $\Theta$ and with $\bar{P}_{0}>0$,  respectively.
	
	\textbf{Step 2.}  Recursively define the accelerated estimate $\hat{\theta}_{k+1}$ based on $\bar{\theta}_{k+1}$ for $k\geq 0$:
		\begin{eqnarray}\label{be2}
			\hat{\theta}_{k+1}&=&\Pi_{P_{k+1}^{-1}}\left\{\hat{\theta}_{k}+ P_{k+1}\phi_{k}\left[\mathcal{H}_{k}(\bar{\theta}_{k})-G_{k}(\phi_{k}^{\top}\bar{\theta}_{k}, \phi_{k}^{\top}\hat{\theta}_{k})\right]\right\},\\
			P_{k+1}&=&P_{k}-a_{k}\beta_{k}P_{k}\phi_{k}\phi_{k}^{\top}P_{k},\;\; a_{k}=\frac{1}{1+\beta_{k}\phi_{k}^{\top}P_{k}\phi_{k}},\label{be26} \\
			\beta_{k}&=&\frac{G_{k}(\phi_{k}^{\top}\bar{\theta}_{k}, \phi_{k}^{\top}\bar{\theta}_{k})-G_{k}(\phi_{k}^{\top}\bar{\theta}_{k}, \phi_{k}^{\top}\hat{\theta}_{k})}{\left(\phi_{k}^{\top}\bar{\theta}_{k}-\phi_{k}^{\top}\hat{\theta}_{k}\right)}I\left(\phi_{k}^{\top}\hat{\theta}_{k}\not=\phi_{k}^{\top}\overline{\theta}_{k}\right)\label{be24} \\
	 &+& \frac{\partial G_{k}(\phi_{k}^{\top}\bar{\theta}_{k}, x)}{\partial x}|_{x=\phi_{k}^{\top}\hat{\theta}_{k}} I\left(\phi_{k}^{\top}\hat{\theta}_{k}=\phi_{k}^{\top}\overline{\theta}_{k}\right),\nonumber
		\end{eqnarray}
	where the initial values $\hat{\theta}_{0}$ and $P_{0}$ can be chosen arbitrarily in $\Theta$ and with $P_{0}>0$, respectively. 
\end{algorithm} 
\begin{remark}
We now provide an explanation for the two steps of our adaptive algorithm. 
In the first step, the scalar adaptation gain $\bar{\beta}_{k}$ is constructed using the bounds $\underline{g}_{k}$ and $\overline{g}_{k}$, which have been utilized in various previous works (see, e.g., \cite{ZZtsqn, ZZ2022, zhao2023}). The design of $\bar{\beta}_{k}$ ensures the convergence of the first step by ensuring the monotonicity of a Lyapunov function related to the estimation error. However, this design of $\bar{\beta}_{k}$ reduces the algorithm's performance because $\bar{\beta}_{k}$ is constructed using only the ``worst-case" information $\underline{g}_{k}$ and $\overline{g}_{k}$. This is the reason why we introduce the second step of the algorithm, where the scalar adaptation gain $\beta_{k}$ in the second step is defined in an adaptive way, by using the preliminary convergent estimate $\bar{\theta}_{k}$ generated in the first step.  We will show that, the matrix gain $P_{k}$ can approach the inverse of the Fisher information, i.e., $\Delta_{k}$. Consequently, the covariance of the estimation error generated by Algorithm \ref{alg2} can reach the well-known C-R bound asymptotically. 
\end{remark}

\begin{remark}
As shown in Algorithm \ref{alg2}, the design of the scalar gain $\bar{\beta}_{k}$ in the first step relies on the bounds $\underline{g}_{k}$ and $\overline{g}_{k}$, which are determined based on the known prior set $\Theta$ specified in Assumption \ref{assum2}. This design ensures that the scalar gain sequence $\{\bar{\beta}_{k}\}$ has a positive lower bound. To further extend this work, the assumption that $\theta$ belongs to a known prior set $\Theta$ may be avoided in the design of $\bar{\beta}_{k}$ by allowing it to gradually decrease to zero. To be concrete, without this assumption, the bounds $\underline{g}_{k}$ and  $\overline{g}_{k}$ in the first-step of the algorithm can be redefined as $$\underline{g}_{k}=\min\limits_{\|x\|\leq \max\{\|\hat{\theta}_{k}\|, d_{k}\}}\frac{\partial G_{k}(\phi_{k}^{\top}\bar{\theta}_{k}, \phi_{k}^{\top}x)}{\partial x},\;\;\;\; \overline{g}_{k}=\max\limits_{\|x\|\leq \max\{\|\hat{\theta}_{k}\|, d_{k}\}}\frac{\partial G_{k}(\phi_{k}^{\top}\bar{\theta}_{k}, \phi_{k}^{\top}x)}{\partial x},$$ where $d_{k}$ is a slowly increasing sequence which can be designed based on the nonlinear function $G_{k}(\cdot, \cdot)$, as specified in Algorithm 1 of \cite{ZZ2024}. With these redefined bounds, the rate at which $\underline{g}_{k}$ approaches zero can be analyzed, enabling the establishment of convergence results for parameter estimates without requiring knowledge of the convex compact set $\Theta$ containing $\theta$ (see, e.g., \cite{ZZ2024}).
\end{remark}

\subsection{Convergence analysis}

In this subsection, we first present the convergence properties of Algorithm \ref{alg2} under general non-PE data conditions in the following theorems:
\begin{theorem}\label{thm6}
	Under Assumptions \ref{assum2}-\ref{assum4}, the estimation error produced by Algorithm \ref{alg2} has the following property as $n\rightarrow \infty$:
	\begin{equation}\label{2300}
		\left\|\theta-\hat{\theta}_{n}\right\|^{2}=O\left(\frac{\log n}{\lambda_{\min}\left\{\sum\limits_{k=1}^{n}\phi_{k}\phi_{k}^{\top}\right\}}\right),\;\;a.s.
	\end{equation}
\end{theorem}
\begin{remark}
	From $(\ref{2300})$, the estimates given by Algorithm \ref{alg2} will be strongly consistent, i.e., $\lim\limits_{n\rightarrow 
		\infty}\|\theta-\hat{\theta}_{n}\|=0,\;a.s.$, if we have
	\begin{equation}\label{2555}
		\log n =o\left(\lambda_{\min}\left\{\sum_{k=1}^{n}\phi_{k}\phi_{k}^{\top}\right\}\right),\;\;n\rightarrow 
		\infty,\;\;a.s. 
	\end{equation}
	Condition $(\ref{2555})$ is much weaker than the well-known PE condition or i.i.d condition, which requires that $n=O\left(\lambda_{\min}\left\{\sum\limits_{k=1}^{n}\phi_{k}\phi_{k}^{\top}\right\}\right)$. Moreover, the condition $(\ref{2555})$ is equal to the $Lai-Wei$ excitation condition, which is known to be the weakest possible data condition for the convergence of the classical least squares (LS) estimates for the linear regression model (cf.,\cite{lai:1982}).
\end{remark}

The next theorem establishes the asymptotic normality of the estimation error. For this, we need a slightly stronger excitation condition than $(\ref{2555})$.
\begin{theorem}\label{thm2}
	Under Assumptions \ref{assum2}-\ref{assum4}, assume that  as $n\rightarrow \infty$,
	\begin{equation}\label{phi}
		(\log n)^{2}=o\left(\lambda_{\min}\left\{\sum_{k=0}^{n}\phi_{k}\phi_{k}^{\top}\right\}\right),\;\;a.s.
	\end{equation}	
	Besides, assume  that
	\begin{equation}\label{R1}
			\Delta_{n}^{\frac{1}{2}}\Lambda_{n}\Delta_{n}^{\frac{1}{2}} \mathop{\rightarrow}\limits^{p} I_{m},\;\;n\rightarrow \infty,			
	\end{equation}	
	where $\Lambda_{n}^{-1}=\sum\limits_{k=0}^{n-1}\lambda_{k}\phi_{k}\phi_{k}^{\top}$, $\lambda_{k}$ is defined as in Lemma $\ref{lem23}$, $\Delta_{n}$ is the Fisher information matrix defined by $\Delta_{n}=\mathbb{E}\left[\sum\limits_{k=0}^{n-1}\lambda_{k}\phi_{k}\phi_{k}^{\top}\right]$ as in Lemma $\ref{lem23}$, $``\mathop{\rightarrow}\limits^{p}"$ means convergence in probability. Thus, the estimates $\hat{\theta}_{n}$ for the parameter $\theta$ given by Algorithm \ref{alg2} have the following asymptotically
	normal property as $n\rightarrow \infty$:
	\begin{equation}\label{4.5}
		\Delta_{n}^{\frac{1}{2}}\tilde{\theta}_{n}\mathop{\rightarrow}\limits^{d} N(0,I_{m}),
	\end{equation}
	where $\tilde{\theta}_{n}=\theta-\hat{\theta}_{n}$, $``\mathop{\rightarrow}\limits^{d}"$ means convergence in distribution. 
\end{theorem}

\begin{remark}
The signal conditions for asymptotic normality required in Theorem \ref{thm2} do not rely on the commonly assumed conditions such as i.i.d., or deterministic and periodic ones found in the literature (see, for example, \cite{guo2015, wang2007, zhao2023, bercu}).
	The condition $(\ref{R1})$ has been considered in \cite{lai:1982} for the analysis of the asymptotic normality of the classical LS estimates, and has also been considered in \cite{ZZtsqn} for the asymptotic normality analysis of nonlinear Newton-type estimates with saturated output observations. We remark that if $\{\phi_{k}\}$ is a deterministic sequence, then $(\ref{R1})$ is satisfied automatically. Moreover, if $\{\phi_{k}\}$ is a random sequence with either ergodic property or $\phi$-mixing property, then under some mild regularity conditions,  $(\ref{R1})$ can also be satisfied. For further illustration, refer to Remark \ref{re11} below for more specific examples. 
\end{remark}
\begin{remark}
Compared to the results presented in \cite{ZZtsqn}, the main difference in our current results lies in the fact that the asymptotic variance of Algorithm \ref{alg2} achieves the C-R bound, whereas the asymptotic variance of the algorithm in \cite{ZZtsqn} cannot reach the C-R bound.
\end{remark}

\subsection{Efficiency analysis}
From Theorem $\ref{thm2}$, it can be concluded that the estimation error follows an asymptotically normal distribution with the covariance being equal to the C-R bound. In this subsection, we will further show that the covariance of estimates produced by Algorithm \ref{alg2} will indeed achieve the C-R bound asymptotically. 
We first provide the $\mathcal{L}_{p}$ convergence property of the estimates.
\begin{assumption}\label{assum5}
For each $r\geq 1$, there exists a constant $c_{r}>0$ such that as $n\rightarrow \infty$,
	\begin{equation}\label{288}
		\left\|\sum_{k=0}^{n}\left(\lambda_{k}\phi_{k}\phi_{k}^{\top}-\mathbb{E}[\lambda_{k}\phi_{k}\phi_{k}^{\top}]\right)\right\|_{2r}\leq c_{r}\sqrt{\lambda_{\min}(n)},
\end{equation}
and 
\begin{equation}\label{delta4}
		n^{\delta}=O\left(\lambda_{\min}(n)\right),\;\; \delta \in \left(\frac{1}{2}, 1\right],
		\end{equation}
		where $\lambda_{\min}(n)=\lambda_{\min}\left\{\mathbb{E}\left[\sum\limits_{k=0}^{n}\phi_{k}\phi_{k}^{\top}\right]\right\}$, $\lambda_{k}$ is defined as in $(\ref{2.9})$, and $\|\cdot\|_{2r}$ is the $\mathcal{L}_{2r}-$norm defined by $\|X\|_{2r}=\left\{\mathbb{E}\|X\|^{2r}\right\}^{\frac{1}{2r}}$.
\end{assumption}

\begin{remark}\label{re11}
	Condition $(\ref{288})$ describes the growth rate of the difference between the information matrix and its expectation. It is a generalization of the $\mathcal{M}_{p}$-class condition in \cite{guo1993} to the case of non-persistent excitation. Assumption $\ref{assum5}$ is weaker than the traditional i.i.d. condition.  Below, we provide some examples of non-i.i.d. signals that satisfy Assumption \ref{assum5} (see \cite{guo1993} for details). The input signals $\{\phi_{k}, k\geq 0\}$ satisfy Assumption $\ref{assum5}$ if either
	\begin{itemize}
		\item[(i)] $\{\phi_{k}, k\geq 0\}$ is a deterministic sequence satisfying non-persistent excitation (non-PE) condition:
	$n^{\delta}=O\left(\lambda_{\min}\left\{\sum\limits_{k=0}^{n}\phi_{k}\phi_{k}^{\top}\right\}\right),\;\; \delta \in \left(\frac{1}{2}, 1\right].$\\
		\item[(ii)] $\{\phi_{k}, k\geq 0\}$ is a bounded $\phi-$mixing process with summable mixing rate $\sum\limits_{i=1}^{\infty} \varphi(i)<\infty$ and with $\inf\limits_{k\geq 0}\left\{\lambda_{\min}\left\{\mathbb{E}\left[\phi_{k}\phi_{k}^{\top}\right]\right\}\right\}>0$. \\
		\item[(iii)] $\left\{\phi_{k}, k\geq 0\right\}$ is bounded and stationary $\alpha-$mixing with $\sum\limits_{t=1}^{\infty}t^{r-1}\left[\alpha(t)\right]^{\delta/(2r+\delta)}$ for some $\delta>0.$
		\item[(iv)] 
		$\{\phi_{k}, k\geq 0\}$ is generated by a trajectory of a dynamical system: 
\begin{equation}\nonumber
\phi_{k+1}=A\phi_{k}+v_{k+1},
\end{equation}
where $\rho(A)< 1$ with
$\rho(A)$ being the spectral radius of state matrix $A$, and the noise sequence $\{v_{k},k\geq 1\}$ is i.i.d. bounded, and satisfies 
$\lambda_{\min}\left\{\mathbb{E}\left[v_{k}v_{k}^{\top}\right]\right\}>0.$ 		
	\end{itemize}
	Furthermore, we mention that adding some regularization terms into the algorithm, as in \cite{ima}, may help validate the conditions on the smallest and largest eigenvalues of the adaptation gain matrix required for convergence.
\end{remark}

\begin{remark}
According to Remark $\ref{re11}$, Assumption $\ref{assum5}$ is much weaker than those used in the existing literature where the regressors $\{\phi_{k}\}$ are required to be either i.i.d. signals (cf., \cite{you2015}), or persistently excited deterministic signals, i.e., satisfying $n=O\left(\lambda_{\min}(n)\right)$ (cf., \cite{wangying2023}).
\end{remark}

\begin{theorem}\label{thm1}
	Under Assumptions $\ref{assum2}$-$\ref{assum4}$ and Assumption $\ref{assum5}$, for each $r\geq 1$, the estimates produced by Algorithm $\ref{alg2}$ have the following convergence rate:
	\begin{equation}\label{15}
			\mathbb{E}\left[\left\|\tilde{\theta}_{n}\right\|^{2r}\right]=O\left(\frac{ \log^{r} n}{n^{r\delta}}\right),\;\;n\rightarrow \infty, 		
	\end{equation}
	where $\tilde{\theta}_{k}=\theta-\hat{\theta}_{k}$, and $\delta$ is defined in Assumption $\ref{assum5}$.	
\end{theorem}

\begin{theorem}\label{thm3}
	Under Assumptions $\ref{assum2}$-$\ref{assum4}$ and Assumption \ref{assum5}, the covariance of estimates produced by Algorithm \ref{alg2} can achieve the C-R bound asymptotically in the following sense:
	\begin{equation}\label{qq}
		\lim_{n\rightarrow \infty}\Delta_{n}^{\frac{1}{2}}Cov(\hat{\theta}_{n})\Delta_{n}^{\frac{1}{2}}=I_{m},
	\end{equation}
	where $\Delta_{n}$ is the Fisher information matrix defined in Lemma \ref{lem23}.
\end{theorem}

\section{Proofs of the main results}\label{sec4} 
Before proving the main results, we first prove some preliminary lemmas.
For simplicity of expression, 
denote 
\begin{eqnarray}
&w_{k+1}&=\mathcal{H}_{k}(\bar{\theta}_{k})-G_{k}(\phi_{k}^{\top}\bar{\theta}_{k}, \phi_{k}^{\top}\theta),	\label{6.245}\\
&\bar{\psi}_{k}&=G_{k}(\phi_{k}^{\top}\bar{\theta}_{k}, \phi_{k}^{\top}\theta)-G_{k}(\phi_{k}^{\top}\bar{\theta}_{k}, \phi_{k}^{\top}\bar{\theta}_{k}),\label{6.22} \\ 
		&\psi_{k}&=G_{k}(\phi_{k}^{\top}\bar{\theta}_{k}, \phi_{k}^{\top}\theta)-G_{k}(\phi_{k}^{\top}\bar{\theta}_{k}, \phi_{k}^{\top}\hat{\theta}_{k}),\label{6.23}
\end{eqnarray}
where the estimates $\bar{\theta}_{k}$ and $\hat{\theta}_{k}$ are defined as in Algorithm $\ref{alg2}$, respectively. From $(\ref{3.5})$, we have 
\begin{equation}\nonumber
		\mathbb{E}\left[w_{k+1}\mid \mathcal{F}_{k}\right]
		=\mathbb{E}\left[\mathcal{H}_{k}(\bar{\theta}_{k})\mid \mathcal{F}_{k}\right]-G_{k}(\phi_{k}^{\top}\bar{\theta}_{k}, \phi_{k}^{\top}\theta)=0,
\end{equation}
thus $\{w_{k+1}, \mathcal{F}_{k}\}$ is a martingale difference sequence. The following lemma gives the bounded property for the conditional moment of $w_{k+1}$.

\begin{lemma}\label{lem51}
	Under Assumption \ref{assum2}-\ref{assum4}, for any given $r\geq 1$, we have	
	\begin{equation}\label{5.4}
		\sup_{k\geq 0}\mathbb{E}\left[\left|w_{k+1}\right|^{r}\mid \mathcal{F}_{k}\right]<C_{r}, 
	\end{equation}
	where $C_{r}$ is some finite positive constant. 	
\end{lemma}	
The proof of Lemma \ref{lem51} is supplied in Appendix \ref{BB}. 

From Assumption $\ref{assum2}$ and Assumption $\ref{assum4}$, the function $G_{k}(y, x)$ is strictly increasing about $x$ for every $y \in \mathbb{R}$, which is established in the next lemma.
\begin{lemma}\label{lem3.1}
		Let Assumptions \ref{assum2}-\ref{assum4} hold, for any constant $\Gamma>0$, the function $G_{k}(y,x)$ defined in $(\ref{3.6})$ has the following property:
	\begin{equation}\label{iinf}
		\inf_{ |y|,|x| \leq \Gamma}\inf _{k\geq 0}\left\{\frac{\partial G_{k}(y,x)}{\partial x}\right\}>0. 
	\end{equation}
	Besides, 
		\begin{eqnarray}
			&\sup_{ |y|,|x| \leq \Gamma}\sup _{k\geq 0}\left\{\left|\frac{\partial G_{k}^{2}(y,x)}{\partial x^{2}}\right|\right\}<\infty, \label{355}\\ 
			&\sup_{ |y|,|x| \leq \Gamma}\sup _{k\geq 0}\left\{\left|\frac{\partial G_{k}^{2}(y,x)}{\partial x \partial y}\right|\right\}<\infty, \label{a43}\\
			&\sup_{ |y|,|x| \leq \Gamma}\sup _{k\geq 0}\left\{\left|\frac{\partial G_{k}^{2}(y,x)}{\partial y^{2}}\right|\right\}<\infty. \label{a44}
		\end{eqnarray}
\end{lemma}

The proof of Lemma \ref{lem3.1} is supplied in Appendix \ref{BB}.

Based on Lemma \ref{lem51} and \ref{lem3.1}, we present the following lemma concerning a stochastic Lyapunov function, which plays a key role in the proof of Theorem \ref{thm6}. The detailed proof of this lemma is provided in Appendix \ref{BB}.
\begin{lemma}\label{lem14}
	Let Assumptions \ref{assum2}-\ref{assum4} hold, then the estimates $\hat{\theta}_{n+1}$ produced by Algorithm \ref{alg2} have the following property as $n \rightarrow \infty$:
		\begin{eqnarray}
			\tilde{\theta}_{n+1}^{\top}P_{n+1}^{-1}\tilde{\theta}_{n+1}+\sum_{k=1}^{n}a_{k}\psi_{k}^{2}=O(\log n),\;\;a.s.,	
		\end{eqnarray}
	where $\tilde{\theta}_{n+1}=\theta-\hat{\theta}_{n}$, besides, $a_{k}$ and $\psi_{k}$ ae  defined by $(\ref{be26})$ and $(\ref{6.23})$ respectively.		
\end{lemma}
			
{\bf Proof of Theorem $\ref{thm6}$.}
From $(\ref{be26})$, we have
\begin{equation}\label{P00}
\begin{aligned}
P_{n+1}^{-1}=P_{n}^{-1}+\beta_{n}\phi_{n}\phi_{n}^{\top}
=P_{0}^{-1}+\sum_{k=0}^{n}\beta_{k}\phi_{k}\phi_{k}^{\top}.
\end{aligned}
\end{equation}
By (\ref{P00}) and (\ref{be24}), we obtain 
	\begin{equation}
		\tilde{\theta}_{n+1}^{\top}P_{n+1}^{-1}\tilde{\theta}_{n+1}
		\geq \min\left\{1, \inf\limits_{k\geq 0}\underline{g}_{k}\right\} \cdot\lambda_{\min}\left\{\sum_{k=0}^{n}\phi_{k}\phi_{k}^{\top}+P_{0}^{-1}\right\}\|\tilde{\theta}_{n+1}\|^{2}, 
	\end{equation}
where $\min\left\{1, \inf\limits_{k\geq 0}\underline{g}_{k}\right\}$ is positive by Lemma $\ref{lem3.1}$. Hence,
Theorem $\ref{thm6}$ follows immediately from Lemma \ref{lem14}.
		
{\bf Proof of Theorem $\ref{thm2}$.}
For simplicity of expression, denote
\begin{equation} \label{ggggg}
	G'_{k}(x)=\frac{\partial G_{k}(\phi_{k}^{\top}\theta,x)}{\partial x},\;\;\;\bar{G}'_{k}(x)=\frac{\partial G_{k}(\phi_{k}^{\top}\bar{\theta}_{k},x)}{\partial x}.
\end{equation}
Besides, let 
\begin{equation}\label{46}
	\mu_{k}=\mathbb{E}\left[w_{k+1}^{2}\mid \mathcal{F}_{k}\right].
\end{equation}	
By the definition of $w_{k+1}$ and $\mu_{k}$ in $(\ref{6.245})$ and $(\ref{46})$ respectively, we have
\begin{equation}\label{4333}
	\begin{aligned}
		\mu_{k}=&\mathbb{E}\left[\mathcal{H}_{k}^{2}(\bar{\theta}_{k})\mid \mathcal{F}_{k}\right]-\left[G_{k}(\phi_{k}^{\top}\bar{\theta}_{k}, \phi_{k}^{\top}\theta)\right]^{2}\\
		=&\mathbb{E}\left[\mathcal{H}_{k}^{2}(\theta)\mid \mathcal{F}_{k}\right]+\mathbb{E}\left[\mathcal{H}_{k}^{2}(\bar{\theta}_{k})-\mathcal{H}_{k}^{2}(\theta)\mid \mathcal{F}_{k}\right]\\
		&+\left[G_{k}(\phi_{k}^{\top}\theta, \phi_{k}^{\top}\theta)\right]^{2}-\left[G_{k}(\phi_{k}^{\top}\bar{\theta}_{k}, \phi_{k}^{\top}\theta)\right]^{2},\\
	\end{aligned}
\end{equation}
where we have used the fact that $G_{k}(\phi_{k}^{\top}\theta, \phi_{k}^{\top}\theta)=0.$ Notice that the probability density function $f(\cdot)$ of the normal distribution satisfies the property  $f'(x)=-\frac{x}{\sigma^{2}}f(x)$. Moreover, by the definitions of $\mathcal{H}_{k}(\cdot)$ and $G_{k}(\cdot)$ in $(\ref{344})$ and $(\ref{3.6})$ respectively, we can verify that
$\mathbb{E}\left[\mathcal{H}_{k}^{2}(\theta)\mid \mathcal{F}_{k}\right]=G'_{k}(\phi_{k}^{\top}\theta).$
Thus, by $(\ref{4333})$ and Lemma \ref{lem3.1}, we can obtain
\begin{equation}\label{433}
	\left|\mu_{k}-G'_{k}(\phi_{k}^{\top}\theta)\right|=O\left(\left|\phi_{k}^{\top}\tilde{\bar{\theta}}_{k}\right|\right),\;\;k\rightarrow \infty.
\end{equation}
Following Theorem \ref{thm6} and the condition $(\ref{phi})$, we have as $k\rightarrow \infty$,
\begin{equation}\label{new}
|\phi_{k}^{\top}\tilde{\bar{\theta}}_{k}|=o(1),\;\;|\phi_{k}^{\top}\tilde{\theta}_{k}|=o(1),\;\; a.s.
\end{equation}
Besides, from the property $(\ref{iinf})$ of function $G_{k}(\cdot, \cdot)$, we have that $G'_{k}(\cdot)$ has a positive lower bound. Therefore, by $(\ref{2.9})$ and $(\ref{433})$, we can obtain
\begin{equation}\label{4300}
	\lim_{k\rightarrow \infty}\mu_{k}/\lambda_{k}=\lim_{k\rightarrow \infty}\mu_{k}/G'_{k}(\phi_{k}^{\top}\theta)= 1,\;\;a.s.
\end{equation}
According to $(\ref{4300})$ and following the same analysis as $[$Theorem 3, \cite{ZZtsqn}$]$, we obtain
\begin{equation}\label{511}
	\Delta_{n}^{-\frac{1}{2}}P_{n}^{-1}\tilde{\theta}_{n}\mathop{\rightarrow}\limits^{d} N(0, I_{m}).
\end{equation}
Moreover, from the definition of $\beta_{k}$ and $(\ref{new})$, we obtain 
	\begin{equation}\label{xxx}
		|\beta_{k}-\lambda_{k}|=|\beta_{k}-G'_{k}(\phi_{k}^{\top}\theta)|=O(|\phi_{k}^{\top}\tilde{\theta}_{k}|+|\phi_{k}^{\top}\tilde{\bar{\theta}}_{k}|)=o(1),\;\;\;k\rightarrow \infty.
			\end{equation}
	Additionally, we have
\begin{equation}\label{xx}
			\begin{aligned}
				\left\|\Delta_{n}^{-\frac{1}{2}}\left(P_{n}^{-1}-\Lambda_{n}^{-1}\right)\Delta_{n}^{-\frac{1}{2}}\right\|
				\leq & \left\|\Delta_{n}^{-\frac{1}{2}}(P_{0}^{-1}+\sum_{k=1}^{m}\phi_{k}\phi_{k}^{\top}|\beta_{k}-\lambda_{k}|)\Delta_{n}^{-\frac{1}{2}}\right\|\\
				&+\left\|\Delta_{n}^{-\frac{1}{2}}(\sum_{k=m+1}^{n}\lambda_{k}\phi_{k}\phi_{k}^{\top})\Delta_{n}^{-\frac{1}{2}}\right\|\cdot \sup_{m+1 \leq k \leq n}\left|\frac{\beta_{k}-\lambda_{k}}{\lambda_{k}}\right|,\;\;a.s.
			\end{aligned}
		\end{equation}
		Letting $n\rightarrow \infty$ and $m\rightarrow \infty$,  it follows from 
$(\ref{R1})$ and $(\ref{xxx})$ that the right-hand side of $(\ref{xx})$ converges to 0 almost surely. Thus, by $(\ref{R1})$ and $(\ref{xx})$, we obtain that
\begin{equation}\label{de23}
\lim_{n\rightarrow \infty}\Delta_{n}^{-\frac{1}{2}}P_{n}^{-1}\Delta_{n}^{-\frac{1}{2}}=I_{m},\;\;a.s.
\end{equation}
Finally, combining $(\ref{511})$ and $(\ref{de23})$, we conclude that
\begin{equation}
	\Delta_{n}^{\frac{1}{2}}\tilde{\theta}_{n}\mathop{\rightarrow}\limits^{d} N(0, I_{m}),\;\;n\rightarrow \infty,
\end{equation}
which completes the proof of Theorem $\ref{thm2}$.

To prove Theorem \ref{thm1}, we present the following lemma.
\begin{lemma}\label{lem6}
	Let Assumptions \ref{assum2}-\ref{assum4} hold, then the estimates $\hat{\theta}_{n+1}$ have the following property for any given $r\geq 1$:
	\begin{equation}\nonumber
			\mathbb{E}\left[\tilde{\theta}_{n+1}^{\top}P_{n+1}^{-1}\tilde{\theta}_{n+1}\right]^{r}+\mathbb{E}\left[\sum_{k=1}^{n}a_{k}\psi_{k}^{2}\right]^{r}
			=O\left(\log^{r} n\right),\;\; n \rightarrow \infty,		
	\end{equation}	
	where $\tilde{\theta}_{n+1}=\theta-\hat{\theta}_{n+1}$, besides, $a_{k}$ and $\psi_{k}$ are  defined by $(\ref{be26})$ and $(\ref{6.23})$, respectively.		
\end{lemma}
The proof of Lemma \ref{lem6} is supplied in Appendix \ref{BB}. 
	
{\bf Proof of Theorem $\ref{thm1}$.}
 By (\ref{2.9}) and Lemma \ref{lem3.1}, we have $\{\lambda_{k}, k\geq 0\}$ has a positive lower bound. Thus,
\begin{equation}\label{53}
	\lambda_{\min}\{\mathbb{E}\left[\Lambda_{n+1}^{-1}\right]\}\geq c \;\lambda_{\min}(n),
\end{equation}	
where $c$ is a constant, $\lambda_{\min}(n)$ is defined in Assumption \ref{assum5}. Besides, From $(\ref{288})$, we have
\begin{equation}\label{54}
	\begin{aligned}
		&\|\tilde{\theta}_{n+1}^{\top}(\Lambda_{n+1}^{-1}-\mathbb{E}[\Lambda_{n+1}^{-1}])\tilde{\theta}_{n+1}\|_{r}\\
		=&O\left(\|\tilde{\theta}_{n+1}\|_{2r}\|\Lambda_{n+1}^{-1}-\mathbb{E}[\Lambda_{n+1}^{-1}]\|_{2r}\right)\\
		=&O\left(\sqrt{\lambda_{\min}(n)}\|\tilde{\theta}_{n+1}\|_{2r}\right),
	\end{aligned}
\end{equation}
where we have used the fact that $\|\tilde{\theta}_{n+1}\|$ is bounded, and $\Lambda_{n+1}^{-1}$ is defined as in Theorem \ref{thm2}. Let $x_{n}=\lambda_{\min}(n)\|\tilde{\theta}_{n+1}\|_{2r}^{2},$ from (\ref{53}), (\ref{54}) and Lemma \ref{lem6}, we have
\begin{equation}\nonumber
	\begin{aligned}
		x_{n}=&O\left(\|\tilde{\theta}_{n+1}^{\top}\mathbb{E}[\Lambda_{n+1}^{-1}]\tilde{\theta}_{n+1}\|_{r}\right)\\
		=&O\left(\|\tilde{\theta}_{n+1}^{\top}\Lambda_{n+1}^{-1}\tilde{\theta}_{n+1}\|_{r}\right)
		+O\left(\|\tilde{\theta}_{n+1}^{\top}\left(\Lambda_{n+1}^{-1}-\mathbb{E}[\Lambda_{n+1}^{-1}]\right)\tilde{\theta}_{n+1}\|_{r}\right)\\
		=&O(\log n)+O\left(\sqrt{x_{n}}\right).
	\end{aligned}
\end{equation}
Therefore, we have 
\begin{equation}\label{55}
x_{n}=\lambda_{\min}(n)\|\tilde{\theta}_{n+1}\|_{2r}^{2}=O(\log n). 
\end{equation}
The proof of Theorem $\ref{thm1}$ is completed by (\ref{55}) and (\ref{delta4}) .

{\bf Proof of Theorem $\ref{thm3}$.}
We first denote
\begin{equation}\label{underg}
\begin{aligned}
	\underline{g}&=\inf_{ |y|,|x| \leq DM}\inf _{k\geq 0}\left\{\frac{\partial G_{k}(y,x)}{\partial x}\right\},\\
	\overline{g}&=\sup_{ |y|,|x| \leq DM}\sup _{k\geq 0}\left\{\frac{\partial G_{k}(y,x)}{\partial x}\right\},
	\end{aligned}
\end{equation}
where $D$ is defined by Assumption $\ref{assum2}$, and 
\begin{equation}\label{mm}
M=\sup\limits_{k\geq 0}\|\phi_{k}\|.
\end{equation}
 From Lemma \ref{lem3.1}, we have $\underline{g}>0$ and $\overline{g}<\infty$. By (\ref{53}) and the condition $(\ref{delta4})$, we have 
\begin{equation}\label{de}
	\|\Delta_{n+1}^{-1}\|=\left\|\mathbb{E}\left[\Lambda_{n+1}^{-1}\right]\right\|=O\left(n^{-\delta}\right),\;\;n\rightarrow \infty,
\end{equation}
where $\frac{1}{2}<\delta\leq 1$. We first prove that
\begin{equation}\label{qp}
	\lim_{n\rightarrow \infty}\mathbb{E}\left[\Delta_{n+1}^{-\frac{1}{2}}P_{n+1}^{-1}\tilde{\theta}_{n+1}\tilde{\theta}_{n+1}^{\top}P_{n+1}^{-1}\Delta_{n+1}^{-\frac{1}{2}}\right]=I_{m}.
\end{equation}	

For each $n\geq 0$, denote
\begin{equation}\label{sos}
	s_{n}=\hat{\theta}_{n}+P_{n+1}\phi_{n}\left[\mathcal{H}_{n}(\bar{\theta}_{n})-G_{n}(\phi_{n}^{\top}\bar{\theta}_{n}, \phi_{n}^{\top}\hat{\theta}_{n})\right].
\end{equation}
From $(\ref{sos})$ and $(\ref{be2})$, we have
\begin{equation}\label{833}
	\begin{aligned}
		\theta-s_{n}
		=&\tilde{\theta}_{n}-P_{n+1}\phi_{n}(\psi_{n}+w_{n+1})\\
		=&(I_{m}-\xi_{n}P_{n+1}\phi_{n}\phi_{n}^{\top})\tilde{\theta}_{n}
		-P_{n+1}\phi_{n}w_{n+1},
	\end{aligned}
\end{equation}	
where
\begin{equation}\nonumber
\begin{aligned}
	 \xi_{n}=
	&\frac{G_{n}(\phi_{n}^{\top}\bar{\theta}_{n}, \phi_{n}^{\top}\theta)-G_{n}(\phi_{n}^{\top}\bar{\theta}_{n}, \phi_{n}^{\top}\hat{\theta}_{n})}{\phi_{n}^{\top}\theta-\phi_{n}^{\top}\hat{\theta}_{n}}I\left(\phi_{n}^{\top}\hat{\theta}_{n}\not=\phi_{n}^{\top}\theta\right)
	+\bar{G}_{n}'(\phi_{n}^{\top}\hat{\theta}_{n})I\left(\phi_{n}^{\top}\hat{\theta}_{n}=\phi_{n}^{\top}\theta\right),
	\end{aligned}
\end{equation}
$w_{n+1}$ and $\psi_{n}$ are defined in $(\ref{6.245})$ and $(\ref{6.23})$, $\bar{G}_{n}'(\cdot)$ is defined in $(\ref{ggggg})$.
Besides, from $(\ref{be2})$, we have
	\begin{equation}\label{855}
		\hat{\theta}_{n+1}=s_{n}I\left(s_{n}\in \Theta \right)+\Pi_{P_{n+1}^{-1}}\{s_{n}\}\cdot I\left(s_{n} \notin \Theta \right)
	\end{equation}
Thus, by $(\ref{be2})$, $(\ref{833})$, and $(\ref{be26})$, we obtain that
\begin{equation}\label{Pr}
	\begin{aligned}
		\tilde{\theta}_{n+1}	=&\theta-s_{n}+\left(s_{n}-\Pi_{P_{n+1}^{-1}}\{s_{n}\}\right)I\left(s_{n}\notin \Theta\right)\\		
		=&(I-\beta_{n}P_{n+1}\phi_{n}\phi_{n}^{\top})\tilde{\theta}_{n}-P_{n+1}\phi_{n}w_{n+1}
		-(\xi_{n}-\beta_{n})P_{n+1}\phi_{n}\phi_{n}^{\top}\tilde{\theta}_{n}\\
		&+(s_{n}-\Pi_{P_{n+1}^{-1}}\{s_{n}\})I\left(s_{n}\notin \Theta\right)\\
		=&P_{n+1}P_{n}^{-1}\tilde{\theta}_{n}-P_{n+1}\phi_{n}w_{n+1}
		-P_{n+1}(\xi_{n}-\beta_{n})\phi_{n}\phi_{n}^{\top}\tilde{\theta}_{n}\\
		&+(s_{n}-\Pi_{P_{n+1}^{-1}}\{s_{n}\})I\left(s_{n}\notin \Theta\right).
	\end{aligned}
\end{equation}
Multiplying both side of $(\ref{Pr})$ by $P_{n+1}$ and proceeding with the recursion, we will have
\begin{equation}\label{888}
	\begin{aligned}
		P_{n+1}^{-1}\tilde{\theta}_{n+1}=&P_{0}^{-1}\tilde{\theta}_{0}-\sum_{k=0}^{n}\phi_{k}w_{k+1}
		-\sum_{k=0}^{n}(\xi_{k}-\beta_{k})\phi_{k}\phi_{k}^{\top}\tilde{\theta}_{k}\\
		&+\sum_{k=0}^{n}P_{k+1}^{-1}\left(s_{k}-\Pi_{P_{k+1}^{-1}}\{s_{k}\}\right)I\left(s_{k}\notin \Theta\right)
	\end{aligned}	
\end{equation}
Thus, we have
\begin{equation}\label{qqq}
	\begin{aligned}
		\mathbb{E}\left[\Delta_{n+1}^{-\frac{1}{2}}P_{n+1}^{-1}\tilde{\theta}_{n+1}\tilde{\theta}_{n+1}^{\top}P_{n+1}^{-1}\Delta_{n+1}^{-\frac{1}{2}}\right]	=\sum_{k=1}^{4}\sum_{t=1}^{4}\mathbb{E}\left[\Delta_{n+1}^{-\frac{1}{2}}A_{k}A_{t}^{\top}\Delta_{n+1}^{-\frac{1}{2}}\right],
	\end{aligned}
\end{equation}	
\begin{equation}\label{a444}
	\begin{aligned}
		A_{1}&=\sum_{k=0}^{n}\phi_{k}w_{k+1},\;\;A_{2}=\tilde{\theta}_{0},\\
		A_{3}&=\sum_{k=0}^{n}(\xi_{k}-\beta_{k})\phi_{k}\phi_{k}^{\top}\tilde{\theta}_{k},\\
		A_{4}&=\sum_{k=0}^{n}P_{k+1}^{-1}\left(s_{k}-\Pi_{P_{k+1}^{-1}}\{s_{k}\}\right)I\left(s_{k}\notin \Theta\right).
	\end{aligned}
\end{equation}

We now analyze the RHS of $(\ref{qqq})$ term by term. Firstly, 
\begin{equation}\label{900}
	\begin{aligned}
		&\mathbb{E}\left[\Delta_{n+1}^{-\frac{1}{2}}A_{1}A_{1}^{\top}\Delta_{n+1}^{-\frac{1}{2}}\right]\\
		=&\mathbb{E}\left[\Delta_{n+1}^{-\frac{1}{2}}\left(\sum_{k=0}^{n}\phi_{k}\phi_{k}^{\top}\mathbb{E}\left[w_{k+1}^{2}\mid\mathcal{F}_{k}\right]\right)\Delta_{n+1}^{-\frac{1}{2}}\right]\\
		=&\mathbb{E}\left[\Delta_{n+1}^{-\frac{1}{2}}\left(\sum_{k=0}^{n}\phi_{k}\phi_{k}^{\top}\mathbb{E}\left[w_{k+1}^{2}\mid\mathcal{F}_{k}\right]-\Lambda_{n+1}^{-1}\right)\Delta_{n+1}^{-\frac{1}{2}}\right]\\
		&+\mathbb{E}\left[\Delta_{n+1}^{-\frac{1}{2}}\Lambda_{n+1}^{-1}\Delta_{n+1}^{-\frac{1}{2}}\right]\\
		=&I_{m}+O\left(\frac{1}{n^{\delta}}\mathbb{E}\left\|\sum_{k=0}^{n}\phi_{k}\phi_{k}^{\top}\mathbb{E}[w_{k+1}^{2}\mid\mathcal{F}_{k}]-\Lambda_{n+1}^{-1}\right\|\right),
	\end{aligned}
\end{equation}
where the matrix $\Lambda_{n+1}^{-1}$ is defined in Theorem \ref{thm2}. 
Let $r=1$ in $(\ref{a20})$ and $(\ref{5.26})$, and from the Cauchy-Schwarz inequality, we can obtain 
\begin{equation} \label{111}
\begin{aligned}
\mathbb{E}\left[\sum_{k=0}^{n}\|\phi_{k}^{\top}\tilde{\bar{\theta}}_{k}\|\right]=&O\left(\sqrt{n\log n}\right),\\\mathbb{E}\left[\sum_{k=0}^{n}\|\phi_{k}^{\top}\tilde{\theta}_{k}\|\right]=&O\left(\sqrt{n\log n}\right).
\end{aligned}
\end{equation}
By $(\ref{433})$ and $(\ref{111})$, we have
\begin{equation}\label{488}
\begin{aligned}
	\mathbb{E}\left[\sum_{k=0}^{n}\|\mu_{k}-G_{k}'(\phi_{k}^{\top}\theta)\|\right]&=\mathbb{E}\left[\sum_{k=0}^{n}\|\phi_{k}^{\top}\tilde{\bar{\theta}}_{k}\|\right]=O\left(\sqrt{n\log n}\right),
	\end{aligned}
\end{equation}
where $\mu_{k}$ is defined in (\ref{46}). Thus, from $(\ref{2.9})$, $(\ref{900})$ and $(\ref{488})$, we have
\begin{equation}\label{71}
	\begin{aligned}
		&\frac{1}{n^{\delta}}\mathbb{E}\left\|\sum_{k=0}^{n}\phi_{k}\phi_{k}^{\top}\mathbb{E}[w_{k+1}^{2}\mid\mathcal{F}_{k}]-\Lambda_{n+1}^{-1}\right\|\\
	\leq&\frac{1}{n^{\delta}}\mathbb{E}\left[\sum_{k=0}^{n}\left\|\left(\frac{\mu_{k}}{G_{k}'(\phi_{k}^{\top}\theta)}-1\right)G_{k}'(\phi_{k}^{\top}\theta)\phi_{k}\phi_{k}^{\top}\right\|\right]\\
		=&O\left(\frac{1}{n^{\delta}}\mathbb{E}\left[\sum_{k=0}^{n}\left\|\phi_{k}^{\top}\tilde{\bar{\theta}}_{k}\right\|\right]\right)\\
		=&O\left(\frac{1}{n^{\delta}}\sqrt{n\log n}\right)=o(1), \text{as}\;\;n\rightarrow \infty.
	\end{aligned}
\end{equation}
Thus, by (\ref{900}) and (\ref{71}), we obtain that
\begin{equation}\label{a1}
	\lim_{n\rightarrow \infty}\mathbb{E}\left[\Delta_{n+1}^{-\frac{1}{2}}A_{1}A_{1}^{\top}\Delta_{n+1}^{-\frac{1}{2}}\right]= I_{m}.
\end{equation}

For the second term of the RHS of $(\ref{qqq})$, since $\|\Delta_{n+1}^{-1}\|\rightarrow 0,$ we can easily obtain that 
\begin{equation}\label{a2}
	\lim_{n\rightarrow \infty}\mathbb{E}\left[\Delta_{n+1}^{-\frac{1}{2}}A_{2}A_{2}^{\top}\Delta_{n+1}^{-\frac{1}{2}}\right]= 0.
\end{equation}

For the third term of the RHS of $(\ref{qqq})$, from Lagrange mean value theorem, for any $k \geq 0$, there exist $\iota_{k} \in \mathbb{R}$ and $\kappa_{k} \in \mathbb{R}$, where $\iota_{k}$ is between $\phi_{k}^{\top}\theta$ and $\phi_{k}^{\top}\hat{\theta}_{k}$, $\kappa_{k}$ is between $\phi_{k}^{\top}\bar{\theta}_{k}$ and $\phi_{k}^{\top}\hat{\theta}_{k}$,  such that $\xi_{k}=\bar{G}_{k}'(\iota_{k})$ and $\beta_{k}=\bar{G}_{k}'(\kappa_{k})$, where $\bar{G}_{k}'(\cdot)$ is defined by $(\ref{ggggg})$. Therefore, from Lemma \ref{lem3.1}, we obtain that 
\begin{equation}\label{73}
	\begin{aligned}
		|\xi_{k}-\beta_{k}|&=|\bar{G}_{k}'(\iota_{k})-\bar{G}_{k}'(\kappa_{k})|= O\left(|\iota_{k}-\kappa_{k}|\right)\\
		&= O\left(\max\left(|\phi_{k}^{\top}\tilde{\bar{\theta}}_{k}|,|\phi_{k}^{\top}\tilde{\theta}_{k}|\right)\right).
	\end{aligned}
\end{equation}
Since $\hat{\mu}_{k}, \beta_{k}, \phi_{k}$ are bounded, by $(\ref{a20})$ and $(\ref{5.26})$, we have
\begin{equation}\label{1000}
	\begin{aligned}
		&\mathbb{E}\left\|\Delta_{n+1}^{-\frac{1}{2}}A_{3}A_{3}^{\top}\Delta_{n+1}^{-\frac{1}{2}}\right\|\\
		= &O\left(\frac{1}{n^{\delta}}\mathbb{E}\left[\left(\sum_{k=0}^{n}\left(\phi_{k}^{\top}\tilde{\theta}_{k}\right)^{2}+\sum_{k=0}^{n}\left(\phi_{k}^{\top}\tilde{\bar{\theta}}_{k}\right)^{2}\right)^{2}\right]\right)\\
		=&O\left(\frac{\log^{2}n}{n^{\delta}}\right)\rightarrow 0,\;\;n \rightarrow \infty.
	\end{aligned}
\end{equation}

We now analyze the fourth term of the RHS of $(\ref{qqq})$. For each $k\geq0$, since $\phi_{k}$ and $\beta_{k}$ are bounded, by $(\ref{be26})$, we have
\begin{equation}\label{10000}
	\left\|P_{n+1}^{-1}\right\|=\left\|\sum_{k=1}^{n}\beta_{k}\phi_{k}\phi_{k}^{\top}+P_{0}^{-1}\right\|=O(n).
\end{equation}
From the property of the projected operator and $(\ref{10000})$, we have
\begin{equation}\label{5.377}
	\begin{aligned}
		&\|P_{k+1}^{-1}(s_{k}-\Pi_{P_{k+1}^{-1}}\{s_{k}\})\|\leq \|P_{k+1}^{-\frac{1}{2}}\|\cdot\|P_{k+1}^{-\frac{1}{2}}(s_{k}-\hat{\theta}_{k})\|\\
		=&\|P_{k+1}^{-\frac{1}{2}}\|\cdot\|P_{k+1}^{-\frac{1}{2}}P_{k+1}\phi_{k}(\psi_{k}+w_{k+1})\|\\
		= &O\left(\sqrt{k}+\sqrt{k}\left|w_{k+1}\right|\right),
	\end{aligned}	
\end{equation}
where we have used the fact that $\phi_{k}$ and $\psi_{k}$ are bounded and $\|P_{k+1}^{\frac{1}{2}}\|\leq \| P_{0}^{\frac{1}{2}}\|$. Moreover, since $\theta$ is in an interior point of $\Theta$, there exists a ball $B(\theta, \epsilon)\subset \Theta$. 
Let $\overline{D}=\frac{\epsilon\underline{g}}{2M\overline{g}}$, where $\underline{g}, \overline{g}, M$ are defined in $(\ref{underg})$ and $(\ref{mm})$, respectively. We then have
\begin{equation}\label{5.38}
	\begin{aligned}
		I\left(s_{k}\notin \Theta\right)
		&\leq I\left(\|\tilde{\theta}_{k}\|>\frac{\epsilon}{2}\right)+I\left(\|P_{k+1}\phi_{k}(w_{k+1}+\psi_{k})\|> \frac{\epsilon}{2}\right)\\
		&\leq I\left(\|\tilde{\theta}_{k}\|>\frac{\epsilon}{2}\right)+I\left(|w_{k+1}+\psi_{k}|> \lambda_{\min}\{P_{k+1}^{-1}\}\overline{D}\right).
	\end{aligned}
\end{equation}
 For simplicity, let 
\begin{equation}\nonumber
\begin{aligned}
Z_{1}=I\left(\|\tilde{\theta}_{k}\|>\frac{\epsilon}{2}\right),\;\;
Z_{2}=I\left(|w_{k+1}+\psi_{k}|> \lambda_{\min}\{P_{k+1}^{-1}\}\overline{D}\right).
\end{aligned}
\end{equation}
Thus, by (\ref{a444}), (\ref{5.377}) and $(\ref{5.38})$, we have 
\begin{equation}\label{5.36}
\setlength\abovedisplayskip{6pt}
\setlength\belowdisplayskip{6pt}
	\begin{aligned}
		&\left\|\mathbb{E}\left[\Delta_{n+1}^{-\frac{1}{2}}A_{4}A_{4}^{\top}\Delta_{n+1}^{-\frac{1}{2}}\right]\right\|\\
		=&O\left(\frac{1}{n^{\delta}}\cdot \mathbb{E}\left[\left(\sum_{k=0}^{n}\left(\sqrt{k}+\sqrt{k}|w_{k+1}|\right)Z_{1}\right)^{2}\right]\right)\\
		&+O\left(\frac{1}{n^{\delta}}\cdot \mathbb{E}\left[\left(\sum_{k=0}^{n}\left(\sqrt{k}+\sqrt{k}|w_{k+1}|\right)Z_{2}\right)^{2}\right]\right).
	\end{aligned}
\end{equation} 

For the first term of the RHS of $(\ref{5.36})$, let $f_{k}=\sqrt{k}I\left(\|\tilde{\theta}_{k}\|>\frac{\epsilon}{2}\right)$, by Theorem \ref{thm1}, we have
\begin{equation}\nonumber
\begin{aligned}
\mathbb{E}\left[\sum_{k=0}^{n}f_{k}^{2}\right]
=&O\left(\sum_{k=0}^{n}kP\left\{\|\tilde{\theta}_{k}\|>\frac{\epsilon}{2}\right\}\right)
		=O\left(\sum_{k=0}^{n}k\mathbb{E}\left[\|\tilde{\theta}_{k}\|^{10}\right]\right)\\
		=&O\left(\sum_{k=0}^{n}\frac{k\log^{5} k}{k^{5\delta}}\right)=O\left(n^{-\frac{1}{2}}\log^{5} n\right),\;\;\text{as} \;\;n\rightarrow \infty.
\end{aligned}
\end{equation}
Thus by Lemma \ref{lee1} in Appendix \ref{BB}, we have as $n\rightarrow \infty,$
\begin{equation}\label{11144}
\begin{aligned}
\left\|\Delta_{n+1}^{-\frac{1}{2}}\right\|^{2}\cdot\mathbb{E}\left[\left(\sum_{k=0}^{n}\left(\sqrt{k}+\sqrt{k}|w_{k+1}|\right)Z_{1}\right)^{2}\right]
=o\left(1 \right).
\end{aligned}
\end{equation}

For the second term of the RHS of $(\ref{5.36})$, from the fact that
\begin{equation}\nonumber
	\lambda_{\min}\{\Delta_{k+1}\}+\lambda_{\min}\{P_{k+1}^{-1}-\Delta_{k+1}\}\leq \lambda_{\min}\{P_{k+1}^{-1}\},
\end{equation}
we have 
\begin{equation}\label{5.45}
\begin{aligned}
	&I\left(|w_{k+1}+\psi_{k}|>\lambda_{\min}\{P_{k+1}^{-1}\}\overline{D}\right)\\
	\leq& I\left(|w_{k+1}+\psi_{k}|>\frac{\lambda_{\min}\{\Delta_{k+1}\}\overline{D}}{3}\right)+I\left(|\lambda_{\min}\{P_{k+1}^{-1}-\Delta_{k+1}^{-1}\}|>\frac{\lambda_{\min}\{\Delta_{k+1}\}}{3}\right).
	\end{aligned}
\end{equation}
For simplicity, let 
\begin{equation}
\begin{aligned}
Z_{3}&=I\left(|w_{k+1}+\psi_{k}|>\frac{\lambda_{\min}\{\Delta_{k+1}\}\overline{D}}{3}\right),\\
Z_{4}&=I\left(|\lambda_{\min}\{P_{k+1}^{-1}-\Delta_{k+1}^{-1}\}|>\frac{\lambda_{\min}\{\Delta_{k+1}\}}{3}\right).
\end{aligned}
\end{equation}
Notice that $\sup\limits_{i\geq 0}\mathbb{E}\left[w_{k+1}^{8}\mid \mathcal{F}_{k}\right]<\infty$ and $\sup\limits_{k\geq 0}|\psi_{k}|<\infty$, by Lemma \ref{lee1}, we have
\begin{equation}\label{55544}
	\begin{aligned}
		&\mathbb{E}\left[\left(\sum_{k=0}^{n}\sqrt{k}|w_{k+1}|Z_{3}\right)^{2}\right]\\
		\leq &\mathbb{E}\left[\left(\sum_{k=0}^{n}\frac{\sqrt{k}}{k^{3\delta}}w_{k+1}^{4}\right)^{2}\right]=n\mathbb{E}\left[\sum_{k=0}^{n}\frac{k}{k^{6\delta}}w_{k+1}^{8}\right]
		=O\left(1\right),
	\end{aligned}
\end{equation}
Besides, notice that by Lemma \ref{lem3.1} and the definition of $\beta_{k}$ in $(\ref{be24})$, we have
\begin{equation}\label{444}
\left\|\beta_{k}-G'_{k}(\phi_{k}^{\top}\theta)\right\|=O\left(\|\phi_{k}^{\top}\tilde{\theta}_{k}\|+\|\phi_{k}^{\top}\tilde{\bar{\theta}}_{k}\|\right),
\end{equation}
then by $(\ref{P00})$, $(\ref{2.9})$ and $(\ref{444})$, we have
\begin{equation}\nonumber
	\begin{aligned}
		\|\Lambda_{n+1}^{-1}-P_{n+1}^{-1}\|
		= &O\left(\sum_{k=0}^{n}\|G_{k}'(\phi_{k}^{\top}\theta)-\beta_{k}\|\right)\\
		= &O\left(\sum_{k=0}^{n}\|\phi_{k}^{\top}\tilde{\theta}_{k}\|+\|\phi_{k}^{\top}\tilde{\bar{\theta}}_{k}\|\right).\\
	\end{aligned}
\end{equation}
Thus, for $\eta=r_{\delta}$, from $(\ref{a20})$ and $(\ref{5.26})$, we have
\begin{equation}\label{5.51}
	\begin{aligned}
		\mathbb{E}\|\Lambda_{n+1}^{-1}-P_{n+1}^{-1}\|^{2\eta}=&O\left(n^{\eta}\log^{\eta}n\right).
	\end{aligned}
\end{equation}
Hence, by $(\ref{5.51})$ and the condition $(\ref{288})$, we have
\begin{equation}\label{58}
	\begin{aligned}
		&\mathbb{E}\|P_{n+1}^{-1}-\Delta_{n+1}\|^{2\eta}\\
		=&O\left(\mathbb{E}\|P_{n+1}^{-1}-\Lambda_{n+1}^{-1}\|^{2\eta}\right)+O\left(\mathbb{E}\|\Lambda_{n+1}^{-1}-\Delta_{n+1}\|^{2\eta}\right)\\
		=&O(n^{\eta}\log^{\eta}n).
	\end{aligned}
\end{equation}
Furthermore, notice that
\begin{equation}\nonumber
	\begin{aligned}
		&\mathbb{E}\left[\sum_{k=0}^{n}\frac{k\|P_{k+1}^{-1}-\Delta_{k+1}\|^{2\eta}}{\left(\lambda_{\min}\{\Delta_{k+1}\}\right)^{2\eta}}\right]\\
		=&O\left(\mathbb{E}\left[\sum_{k=0}^{n}\frac{k\cdot (k\log k)^{\frac{4-2\delta}{2\delta-1}}}{k^{\frac{8\delta-4\delta^{2}}{2\delta-1}}}\right]\right)
		=O(1), n\rightarrow \infty,
	\end{aligned}
\end{equation}
and
\begin{equation}\label{1144}
	\begin{aligned}
		&\mathbb{E}\left[\left(\sum_{k=0}^{n}\frac{\sqrt{k}\|P_{k+1}^{-1}-\Delta_{k+1}\|^{\eta}}{\left(\lambda_{\min}\{\Delta_{k+1}\}\right)^{\eta}}\right)^{2}\right]\\
		=&\mathbb{E}\left[\left(\sum_{k=0}^{n}\frac{k^{1+\frac{\eta}{2}}}{\left(\lambda_{\min}\{\Delta_{k+1}\}\right)^{\eta}}\right)\left(\sum_{k=0}^{n}\frac{\|P_{k+1}^{-1}-\Delta_{k+1}\|^{2\eta}}{\left(\lambda_{\min}\{\Delta_{k+1}\}\right)^{\eta}k^{\frac{\eta}{2}}}\right)\right]\\
		=&o(n^{\delta}), \;\;\;\;\; n\rightarrow \infty.\\
	\end{aligned}
\end{equation}
 Thus, from Lemma \ref{lee1} in Appendix \ref{BB}, we have 
\begin{equation}\label{5.57}
	\begin{aligned}
		&\mathbb{E}\left[\left(\sum_{k=0}^{n}\left(\sqrt{k}+\sqrt{k}|w_{k+1}|\right)Z_{4}\right)^{2}\right]\\
		\leq &\mathbb{E}\left[\left(\sum_{k=0}^{n}\frac{\sqrt{k}|w_{k+1}|\|P_{k+1}^{-1}-\Delta_{k+1}\|^{\eta}}{\left(\lambda_{\min}\{\Delta_{k+1}\}\right)^{\eta}}\right)^{2}\right]=o(n^{\delta}), n\rightarrow \infty.
\end{aligned}
\end{equation}
Hence, by $(\ref{5.45})$, $(\ref{55544})$, and $(\ref{5.57})$, we have
\begin{equation}\label{4411}
	\frac{1}{n^{\delta}}\cdot \mathbb{E}\left[\left(\sum_{k=0}^{n}\left(\sqrt{k}+\sqrt{k}|w_{k+1}|\right)Z_{2}\right)^{2}\right]=o\left(1\right).
\end{equation}
Finally, by $(\ref{5.36})$, $(\ref{11144})$, and $(\ref{4411})$, we obtain  
\begin{equation}\label{12343}
	\lim_{n\rightarrow \infty}\left\|\mathbb{E}\left[\Delta_{n+1}^{-\frac{1}{2}}A_{4}A_{4}^{\top}\Delta_{n+1}^{-\frac{1}{2}}\right]\right\|=0. 
\end{equation}
Therefore, by $(\ref{a1})$, $(\ref{a2})$, $(\ref{1000})$, and $(\ref{12343})$, we arrive at $(\ref{qp})$. 

Let $\psi_{n+1}=(\Delta_{n+1}^{\frac{1}{2}}-\Delta_{n+1}^{-\frac{1}{2}}P_{n+1}^{-1})\tilde{\theta}_{n+1},$ by $(\ref{58})$ and Theorem \ref{thm1}, we have
\begin{equation}\label{5.65}
\setlength\abovedisplayskip{6pt}
\setlength\belowdisplayskip{6pt}
	\begin{aligned}
		&\left\|\mathbb{E}\left[\psi_{n+1}\psi_{n+1}^{\top}\right]\right\|\\
		=& O\left(\left\|\Delta_{n+1}^{-\frac{1}{2}}\right\|^{2}\mathbb{E}\left[\|\Delta_{n+1}-P_{n+1}^{-1}\|^{2}\|\tilde{\theta}_{n+1}\|^{2}\right]\right)\\
		=&O\left(\frac{1}{n^{\delta}}\|\Delta_{n+1}-P_{n+1}^{-1}\|_{4}^{2}\cdot\|\tilde{\theta}_{n+1}\|_{4}^{2}\right)\\
		=&O\left(\frac{n\log n}{n^{\delta}} \sqrt{\frac{\log^{2}n}{n^{2\delta}}}\right)=o(1),\;\;n\rightarrow \infty.
	\end{aligned}	
\end{equation}
From $(\ref{qp})$ and $(\ref{5.65})$, we can obtain
\begin{equation}\label{5.61}
\setlength\abovedisplayskip{6pt}
\setlength\belowdisplayskip{6pt}
	\begin{aligned}
		\lim_{n\rightarrow \infty}\mathbb{E}\left[\Delta_{n+1}^{\frac{1}{2}}\tilde{\theta}_{n+1}\tilde{\theta}_{n+1}^{\top}\Delta_{n+1}^{\frac{1}{2}}\right]=I_{m}.
	\end{aligned}
\end{equation}

For the asymptomatic property of $\Delta_{n}^{\frac{1}{2}}\mathbb{E}\left[\tilde{\theta}_{n}\right]$. From $(\ref{888})$ we have $\mathbb{E}\left[\Delta_{n}^{-\frac{1}{2}}A_{1}\right]=0.$ Thus,
by $(\ref{a2})$, $(\ref{1000})$ and $(\ref{12343})$, we have
\begin{equation}\label{1312}
	\lim_{n\rightarrow \infty}\mathbb{E}\left[\Delta_{n}^{-\frac{1}{2}}P_{n}^{-1}\tilde{\theta}_{n}\right]=\mathbb{E}\left[\Delta_{n}^{-\frac{1}{2}}\left(A_{2}+A_{3}+A_{4}\right)\right]=0.
\end{equation}
Thus, by $(\ref{1312})$ and the similar analysis as in $(\ref{5.65})$, we have
\begin{equation}\label{0917}
	\lim_{n\rightarrow \infty}\Delta_{n}^{\frac{1}{2}}\mathbb{E}\left[\tilde{\theta}_{n}\right]=0.
\end{equation}
Finally, $(\ref{qq})$ is obtained by $(\ref{5.61})$ and $(\ref{0917})$.

\section{Numerical simulation}\label{sec5}
The goal of this section is to provide a simulation example to demonstrate the theoretical results obtained in this paper. For that purpose, we will compare the convergence speeds and estimation error variance of our current asymptotically efficient algorithm (Algorithm \ref{alg2}) with two different algorithms:  the SN algorithm $(\ref{sn})$ and also with the TSQN algorithm in \cite{ZZtsqn}.

Consider the stochastic systems $(\ref{eq1})$-$(\ref{eq2})$ with the thresholds $L_{k}=l_{k}=0$, $U_{k}=u_{k}=8$ for each $k\geq 0$; the parameter $\theta=\left[-1.2, 0.5, 1, -0.5, 1.5, -1, 1.8, 0.8, -2, 0.4, 1\right]^{\top}$. Besides, the noise sequence $\{v_{k+1}\}$ is i.i.d with normal distribution $N(0, 1)$. The regressor $\phi_{k}=[\varphi_{k}, 5]^{\top}$, and $\{\varphi_{k}, k\geq 0\}\in \mathbb{R}^{10}$ are generated by the following stochastic dynamical system:
\begin{equation}\nonumber
		\phi_{k+1}=A^{\top}\phi_{k}+u_{k+1},\;\;k=0,1,\cdots
\end{equation}
where $\phi_{0}=0$, the state matrix $A=\diag\left[0.18, 0.4, 0.63, 0.82, 0.5, 0.47, 0.41, 0.08, 0.56, 0.45\right]^{\top}$;  the input $u_{k}=(u_{k}^{(1)}, u_{k}^{(2)}, \cdots, u_{k}^{(10)})$, $u_{k}^{(i)}\sim U[-1,1], 1\leq i\leq 10, k\geq 0$. Since the regressions $\{\phi_{k}\}$ are generated by a dynamical system, it does not satisfy the independence or periodicity conditions assumed in previous literature, and it can be verified that Assumption $\ref{assum5}$ is satisfied.
To estimate $\theta$, let the initial value $\hat{\theta}_{0}=\bar{\theta}_{0}=[0, 0,\cdots, 0]^{\top}$, and the convex compact parameter set be $\Theta=\{x\in \mathbb{R}^{10}: |x^{(i)}|\leq 2 \},$ on which the estimates are projected.

We will compare the performance of algorithms based on the mean square error $M_{k}$ and the covariance $R_{k}$ of the estimation error, which are defined as follows:
\begin{equation}
\begin{aligned}
M_{k}&=\mathbb{E}\left[\left\|\theta-\hat{\theta}_{k}\right\|^{2}\right],\\
R_{k}&=\mathbb{E}\left[\left(\hat{\theta}_{k}-\mathbb{E}\left[\hat{\theta}_{k}\right]\right)\left(\hat{\theta}_{k}-\mathbb{E}\left[\hat{\theta}_{k}\right]\right)^{\top}\right].
\end{aligned}
\end{equation}
Here, we repeat the simulation for $m = 500$ times with a maximum number of iterations $n = 5000.$  Then we can get $m$ sequences $\left\{\|\theta-\hat{\theta}_{k}^{j}\|^2, 1\leq j \leq m, 1\leq k \leq n\right\}$, where the superscript $j$
denotes the $j$th simulation result. We then use
$\frac{1}{m}\sum\limits_{j=1}^{m}\left\|\theta-\hat{\theta}_{k}^{j}\right\|^2$ and  $\frac{1}{m}\sum\limits_{j=1}^{m}\left\|\hat{\theta}_{k}^{j}-\frac{1}{m}\sum\limits_{j=1}^{m}\hat{\theta}_{k}^{j}\right\|^2$ to approximate $M_{k}$ and the trace of  $R_{k}$, respectively. 

By $(\ref{15})$ in Theorem $\ref{thm1}$, the MSE of Algorithm $\ref{alg2}$ will converge to $0$, which is verified by the trajectory of $M_{k}$ in Figure \ref{fig1}. Moreover, it can be shown that the current asymptotically efficient Algorithm \ref{alg2} outperforms the other two algorithms in terms of the convergence speed of the MSE. 
\begin{figure}[!htbp]
	\centering
\includegraphics[width=0.7\linewidth]{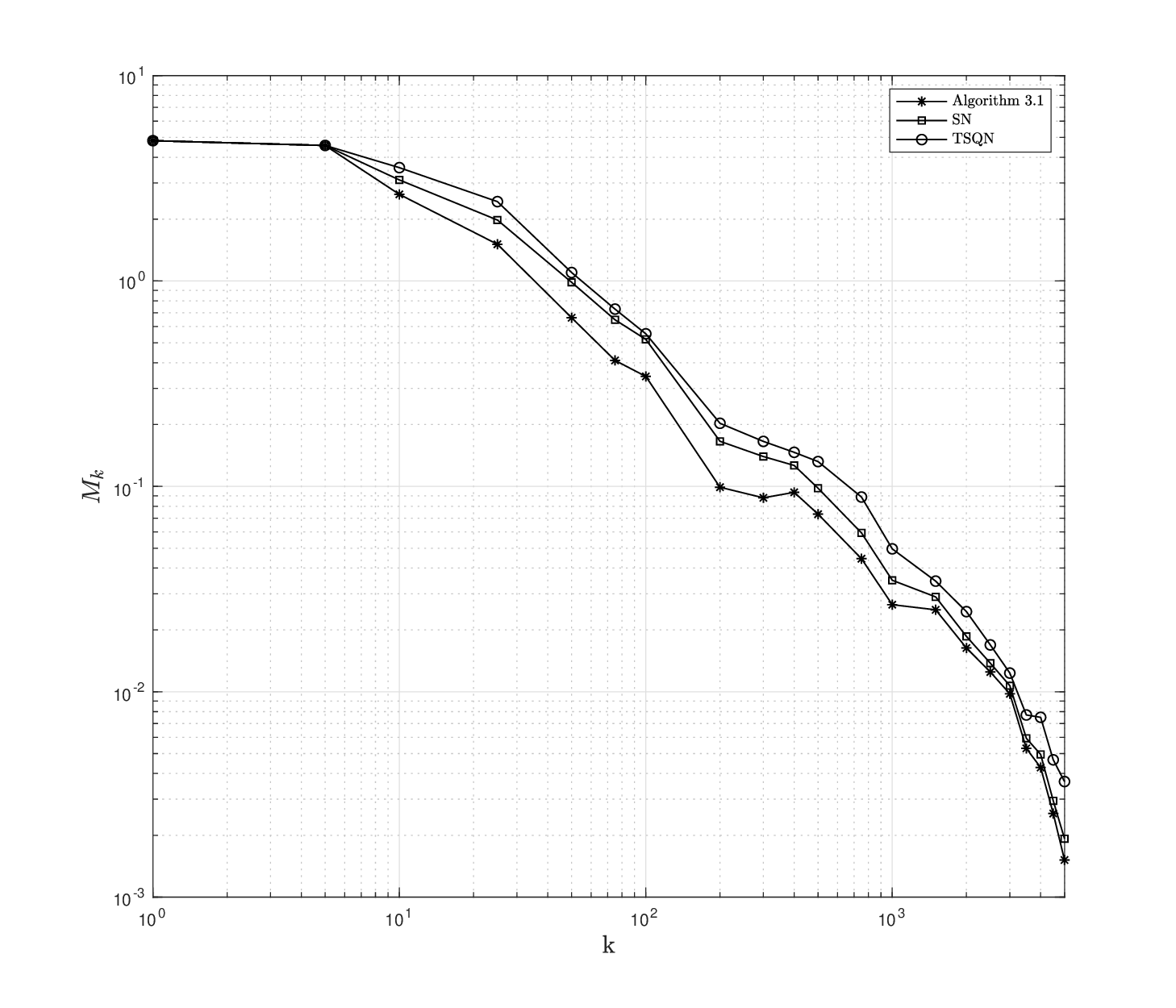}
		\caption{Mean square error}
		\label{fig1}
		\end{figure}

Moreover, it can be easily verified that $ tr(\Delta_{k}^{-1}(\theta))\sim\frac{1}{k}$ as $k\rightarrow \infty$.
Figure \ref{fig2} shows the trajectory of the $k\cdot tr(R_{k})$, demonstrating that the trajectory of the $k\cdot tr(R_{k})$ of the current asymptotically efficient algorithm can converge to the CR bound, i.e., $k\cdot tr(\Delta_{k}^{-1}(\theta))$, achieving comparable convergence performance to the SN algorithm. Besides, the covariance of the TSQN algorithm is larger than that of the current algorithm. This is because the TSQN algorithm is based on classical least squares, where the adaptation gain matrix does not equal the Fisher information matrix in the current nonlinear observation case.
\begin{figure}[!htbp]
		\centering
		\includegraphics[width=0.7\linewidth]{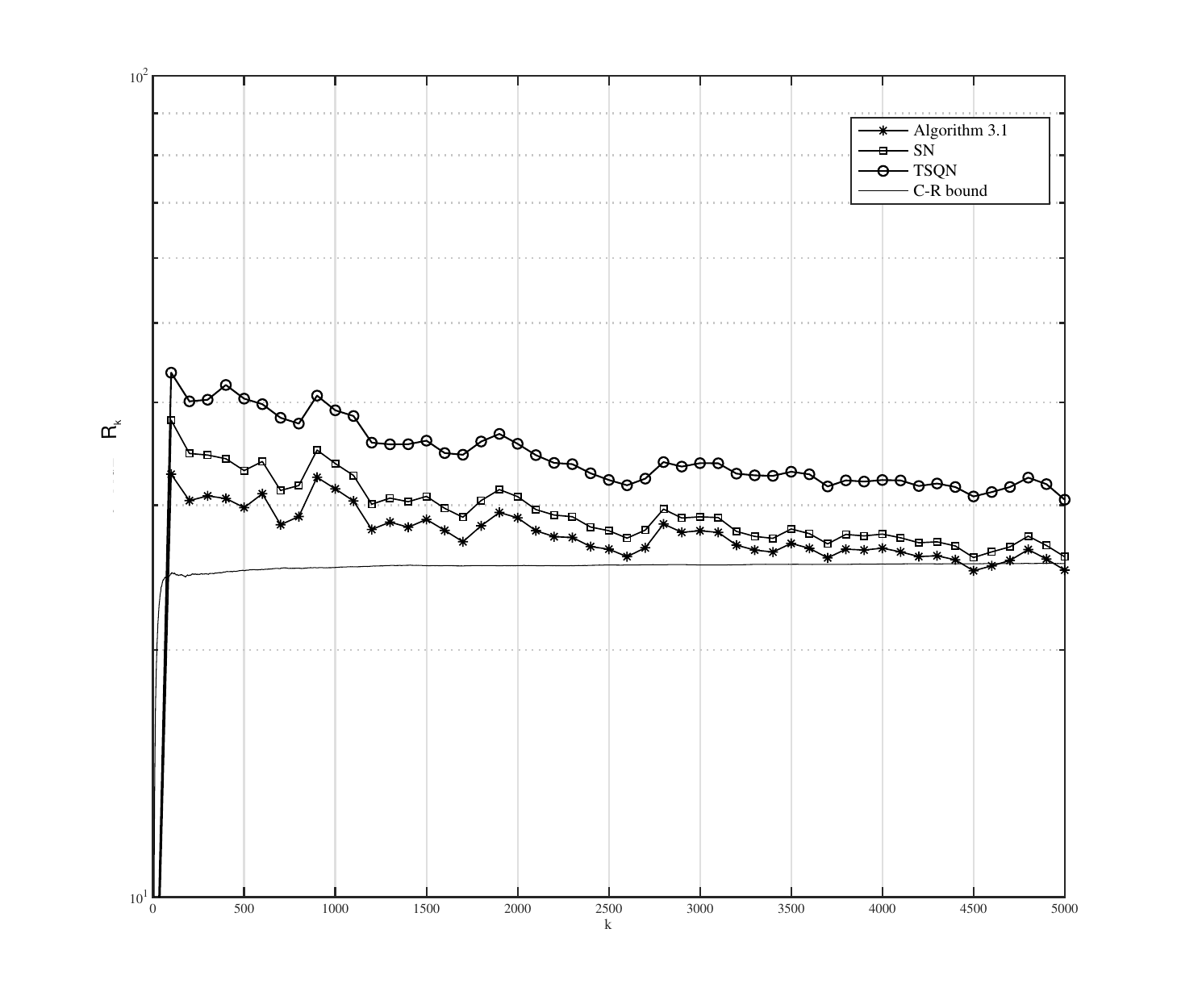}
		\caption{Covariance and CR lower bound}
		\label{fig2}
		\end{figure}

\section{Concluding remarks}\label{sec6}
In this paper, we have investigated the efficient estimation problems for stochastic regression models with saturated output observations. We have proposed a new two-step adaptive algorithm to estimate the unknown parameters under non-i.i.d data conditions. It is shown that the estimates are strongly consistent and asymptotically normal, even under general non-PE excitation conditions. It is also shown that the  MSE of the estimates can asymptotically approach the celebrated C-R bound, indicating that the performance of the proposed algorithm is the best possible that one can expect in general.  We remark that none of the results presented in this paper requires strong assumptions on regressors that are commonly used in previous literature, such as the independence or periodicity conditions, and thus does not exclude their applications to stochastic feedback systems. For future investigation, there are still several interesting problems that need to be solved, for example,  how to establish global convergence results with non-i.i.d data for more complicated stochastic nonlinear regression models including multi-layer neural networks, and how to combine adaptive learning with feedback control for stochastic nonlinear dynamical systems, etc.	

\begin{appendix}

\section{Lemmas and proofs}\label{BB} %
\begin{lemma}\label{lem5} (\cite{ce2001}). The projection operator given by Definition $\ref{def2}$  satisfies
	\begin{equation}\label{17}
		\|\Pi_{A}\{x\}-\Pi_{A}\{y\}\|_{A} \leq \|x-y\|_{A}\quad \forall x, y\in \mathbb{R}^{m}
	\end{equation}
\end{lemma}

\begin{lemma}\label{lem3} (\cite{lai:1982}). Let \;$X_{1}, X_{2},\cdots$ be a sequence of vectors in $	\mathbb{R}^{m} (m\geq 1)$ and let $A_{n} = A_{0}+\sum_{i=1}^{n}X_{i}X_{i}^{\top}$. Let $|A_{n}|$ denote the determinant of $A_{n}$. Assume that $A_{0}$ is nonsingular, then as $n\rightarrow \infty$
	\begin{equation}\label{20}
		\sum_{k=1}^{n}\frac{X_{k}^{\top}A_{k-1}^{-1}X_{k}}{1+X_{k}^{\top}A_{k-1}^{-1}X_{k}} = O(\log(|A_{n}|)).	
	\end{equation}
\end{lemma}
\begin{lemma}\label{lem1} (\cite{y2003}).
	If $f=\{f_{n}, n\geq 1\}$ is an $\mathcal{L}_{1}$ martingale and $p\in (1,\infty)$, there exists 
	\begin{equation}\nonumber
		\mathbb{E}\left[f_{n}\right]^{p}\leq 18p^{\frac{3p}{2}}(p-1)^{\frac{p}{2}}\mathbb{E}\left[ \sum_{j=1}^{n}(f_{j}-f_{j-1})^{2}\right]^{\frac{p}{2}}
	\end{equation}
\end{lemma}

\begin{lemma}\label{lem2} (\cite{chen:1991}). Let $\left\{w_{n}, \mathcal{F}_{n}\right\}$ be a martingale difference sequence and $\left\{f_{n}, \mathcal{F}_{n}\right\}$ an adapted sequence.
		If $\sup_{n} 	\mathbb{E}[|w_{n+1}|^{\alpha}\mid \mathcal{F}_{n}] < \infty, a.s.$, for some $\alpha \in (0, 2]$, then as $n\rightarrow \infty$:
		\begin{equation}\label{19}
			\sum_{i=0}^{n}f_{i}w_{i+1} = O(s_{n}(\alpha)\log^{\frac{1}{\alpha}+\eta}(s_{n}^{\alpha}(\alpha)+e))\;a.s., \forall \eta >0,
		\end{equation}
		where $s_{n}(\alpha)=\left(\sum\limits_{i=0}^{n}|f_{i}|^{\alpha}\right)^{\frac{1}{\alpha}}$.
	\end{lemma}

\begin{lemma}\label{lee1}
If $\left\{f_{n}, \mathcal{F}_{n}\right\}$ is an adapted sequence and $\sup\limits_{n\geq 0}\mathbb{E}[|w_{n+1}|^{2r}\mid \mathcal{F}_{n}] < \sigma_{r} $ for some $r\geq 1$, where $\sigma_{r}$ is a constant, then as $n\rightarrow \infty$,
		\begin{equation}\label{5663}
			\mathbb{E}\left[\left(\sum_{k=0}^{n}f_{k}|w_{k+1}|^{r}\right)^{2}\right]=O\left(\mathbb{E}\left[s_{n}^{2}(2)+s_{n}^{2}(1)\right]\right),
		\end{equation}
		where $s_{n}(\alpha)=\left(\sum\limits_{k=0}^{n}|f_{k}|^{\alpha}\right)^{\frac{1}{\alpha}}$.
\end{lemma}
{\bf Proof of Lemma \ref{lee1}}: 
Notice that
\begin{equation}\label{463}
	\begin{aligned}
		\mathbb{E}\left[\left(\sum_{k=0}^{n}f_{k}|w_{k+1}|^{r}\right)^{2}\right]
		=&O\left(\mathbb{E}\left[\left(\sum_{k=0}^{n}f_{k}\left(|w_{k+1}|^{r}-\mathbb{E}\left[|w_{k+1}|^{r}\mid \mathcal{F}_{k}\right]\right)\right)^{2}\right]\right)\\
		&+O\left(\mathbb{E}\left[\left(\sum_{k=0}^{n}f_{k}\mathbb{E}\left[|w_{k+1}|^{r}\mid \mathcal{F}_{k}\right]\right)^{2}\right]\right).
	\end{aligned}
\end{equation}
By Burkholder equality (see, Lemma \ref{lem1}) and Minkowski equality, we obtain
\begin{equation}
	\begin{aligned}
	\mathbb{E}\left[\left(\sum_{k=0}^{n}f_{k}\left(|w_{k+1}|^{r}-\mathbb{E}\left[|w_{k+1}|^{r}\mid \mathcal{F}_{k}\right]\right)\right)^{2}\right]
	=O\left(\mathbb{E}\left[\sum_{k=0}^{n}f_{k}^{2}\right]\right),\;\; n\rightarrow \infty,
	\end{aligned}
\end{equation}
and
\begin{equation}
	\mathbb{E}\left[\left(\sum_{k=0}^{n}f_{k}\mathbb{E}\left[|w_{k+1}|^{r}\mid \mathcal{F}_{k}\right]\right)^{2}\right]=O\left(\mathbb{E}\left[\left(\sum_{k=0}^{n}\left|f_{k}\right|\right)^{2}\right]\right),\;\; n\rightarrow \infty.
\end{equation}
Thus, Lemma \ref{lee1} is obtained.

{\bf Proof of Lemma $\ref{lem14}$.}
Following the analysis ideas of the classical least-squares for linear stochastic regression models (see e.g., \cite{moore:1978, lai:1982, guo1995}), we first investigate the properties of the stochastic Lyapunov function, given by: 
\begin{equation}\label{vvvv}
V_{n}=\tilde{\theta}_{n+1}^{\top}P_{n+1}^{-1}\tilde{\theta}_{n+1}.
\end{equation}
	Notice that by the definition of the function $G_{k}(\cdot, \cdot)$ in $(\ref{3.6})$, we have $$G_{k}(\phi_{k}^{\top}\bar{\theta}_{k}, \phi_{k}^{\top}\bar{\theta}_{k})=0.$$ Besides, by $(\ref{6.23})$ and the definition of $\beta_{k}$ in $(\ref{be26})$-$(\ref{be24})$,  we know that
	\begin{equation}\nonumber
		\begin{aligned}
			\bar{\psi}_{k}=&G_{k}(\phi_{k}^{\top}\bar{\theta}_{k}, \phi_{k}^{\top}\theta)-G_{k}(\phi_{k}^{\top}\bar{\theta}_{k},\phi_{k}^{\top}\bar{\theta}_{k})\\
			=&G_{k}(\phi_{k}^{\top}\bar{\theta}_{k},\phi_{k}^{\top}\theta)-G_{k}(\phi_{k}^{\top}\bar{\theta}_{k},\phi_{k}^{\top}\hat{\theta}_{k})+G_{k}(\phi_{k}^{\top}\bar{\theta}_{k}, \phi_{k}^{\top}\hat{\theta}_{k})-G_{k}(\phi_{k}^{\top}\bar{\theta}_{k}, \phi_{k}^{\top}\bar{\theta}_{k})\\
			=&\psi_{k}+\beta_{k}\phi_{k}^{\top}(\hat{\theta}_{k}-\bar{\theta}_{k})
			=\psi_{k}+\beta_{k}\phi_{k}^{\top}(\theta-\bar{\theta}_{k}-\tilde{\theta}_{k}).
		\end{aligned}
	\end{equation}
	Hence,
	\begin{equation}\label{barpsi}
		\psi_{k}-\beta_{k}\phi_{k}^{\top}(\theta-\hat{\theta}_{k})=\bar{\psi}_{k}-\beta_{k}\phi_{k}^{\top}(\theta-\bar{\theta}_{k}).
	\end{equation}
	Moreover, by  $(\ref{be26})$,  we can obtain that
	\begin{equation}\label{P-1-1}
		P_{k+1}^{-1}=P_{k}^{-1}+\beta_{k}\phi_{k}\phi_{k}^{\top}.
	\end{equation}
	Hence, multiplying $a_{k}P_{k}\phi_{k}$ from the left hand side of $(\ref{P-1-1})$ and noticing the definition of $a_{k}$, we know that
		\begin{eqnarray}\label{aalpha}
			a_{k}P_{k+1}^{-1}P_{k}\phi_{k}
			=a_{k}\phi_{k}(1+\beta_{k}\phi_{k}^{\top}P_{k}\phi_{k})=\phi_{k}.
		\end{eqnarray}
	From $(\ref{be2})$, Lemma \ref{lem5}, $(\ref{barpsi})$ and $(\ref{aalpha})$, we have
	\begin{equation}\label{61}
		\begin{aligned}
			V_{k+1}\leq&[\tilde{\theta}_{k}-P_{k+1}\phi_{k}(\psi_{k}+w_{k+1})]^{\top}P_{k+1}^{-1}
			[\tilde{\theta}_{k}-P_{k+1}\phi_{k}(\psi_{k}+w_{k+1})]\\
			=&V_{k}+\beta_{k}(\phi_{k}^{\top}\tilde{\theta}_{k})^{2}-2\phi_{k}^{\top}\tilde{\theta}_{k}\psi_{k}
			+\phi_{k}^{\top}P_{k+1}\phi_{k}\psi_{k}^{2}\\
			&-2\phi_{k}^{\top}\tilde{\theta}_{k}w_{k+1}
			+2\phi_{k}^{\top}P_{k+1}\phi_{k}\psi_{k}w_{k+1}\\
			&+\phi_{k}^{\top}P_{k+1}\phi_{k}w_{k+1}^{2}\\
			=&V_{k}-\beta_{k}^{-1}\psi_{k}^{2}+\beta_{k}^{-1}(\psi_{k}-\beta_{k}\phi_{k}^{\top}\tilde{\theta}_{k})^{2}
			+a_{k}\phi_{k}^{\top}P_{k}\phi_{k}\psi_{k}^{2}\\&+2\beta_{k}^{-1}(\psi_{k}-\beta_{k}\phi_{k}^{\top}\tilde{\theta}_{k})w_{k+1}
			-2\beta_{k}^{-1}\psi_{k}w_{k+1}
			+2a_{k}\phi_{k}^{\top}P_{k}\phi_{k}\psi_{k}w_{k+1}\\ &+a_{k}\phi_{k}^{\top}P_{k}\phi_{k}w_{k+1}^{2}\\
			=& V_{k}-a_{k}\beta_{k}^{-1}\psi_{k}^{2}+\beta_{k}^{-1}(\bar{\psi}_{k}-\beta_{k}\phi_{k}^{\top}\tilde{\bar{\theta}}_{k})^{2}-2a_{k}\beta_{k}^{-1}\psi_{k}w_{k+1}\\
			 &+2\beta_{k}^{-1}(\bar{\psi}_{k}-\beta_{k}\phi_{k}^{\top}\tilde{\bar{\theta}}_{k})w_{k+1}+a_{k}\beta_{k}\phi_{k}^{\top}P_{k}\phi_{k}w_{k+1}^{2},\\
		\end{aligned}
	\end{equation}
	where $w_{k+1}$ is defined in $(\ref{6.245})$. Summing up both sides of $(\ref{61})$ from $0$ to $n$, we have
	\begin{equation}\label{s1u}
		\begin{aligned}
			V_{n+1}&\leq V_{0}-\sum_{k=0}^{n}a_{k}\beta_{k}^{-1}\psi_{k}^{2}+\sum_{k=0}^{n}\beta_{k}^{-1}(\bar{\psi}_{k}-\beta_{k}\phi_{k}^{\top}\tilde{\bar{\theta}}_{k})^{2}-\sum_{k=0}^{n}2a_{k}\beta_{k}^{-1}\psi_{k}w_{k+1}\\	
			&+\sum_{k=0}^{n}2\beta_{k}^{-1}(\bar{\psi}_{k}-\beta_{k}\phi_{k}^{\top}\tilde{\bar{\theta}}_{k})w_{k+1}
			+\sum_{k=0}^{n}a_{k}\beta_{k}\phi_{k}^{\top}P_{k}\phi_{k}\left[w_{k+1}^{2}-\mathbb{E}\left[w_{k+1}^{2}\mid \mathcal{F}_{k}\right]\right]\\&+\sum_{k=0}^{n}a_{k}\beta_{k}\phi_{k}^{\top}P_{k}\phi_{k}\mathbb{E}\left[w_{k+1}^{2}\mid \mathcal{F}_{k}\right].
		\end{aligned}
	\end{equation}
	
	We now analyze the right-hand side of $(\ref{s1u})$ term by term. Firstly, for the last term of $(\ref{s1u})$, let $X_{k}=\beta_{k}^{\frac{1}{2}}\phi_{k}$ in Lemma \ref{lem3}, and by Lemma $\ref{lem51}$, we get 
	\begin{equation}\label{43333}
		\sum_{k=0}^{n}a_{k}\beta_{k}\phi_{k}^{\top}P_{k}\phi_{k}\mathbb{E}\left[w_{k+1}^{2}\mid \mathcal{F}_{k}\right]=O\left(\log n\right),\;a.s.
	\end{equation}
	For other noise terms of $(\ref{s1u})$, by Lemma \ref{lem51} and using the martingale estimation theorem (see, Lemma \ref{lem2} in Appendix \ref{BB}), we have
	\begin{equation}\label{433333}
		\begin{aligned}
			\sum_{k=0}^{n}a_{k}\beta_{k}^{-1}\psi_{k}w_{k+1}&=o\left(\sum_{k=0}^{n}a_{k}\psi_{k}^{2}\right),\;\;a.s.,\\ 
			\sum_{k=0}^{n}\beta_{k}^{-1}(\bar{\psi}_{k}-\beta_{k}\phi_{k}^{\top}\tilde{\bar{\theta}}_{k})w_{k+1}&=o\left(\sum_{k=0}^{n}\left(\bar{\psi}_{k}-\beta_{k}\phi_{k}^{\top}\tilde{\bar{\theta}}_{k}\right)^{2}\right),\;a.s.,\;\;\;\;\;\;\;\;\\ 
		\sum_{k=0}^{n}a_{k}\beta_{k}\phi_{k}^{\top}P_{k}\phi_{k}\left[w_{k+1}^{2}-\mathbb{E}_{k}\left[w_{k+1}^{2}\right]\right]&=o\left(\sum_{k=0}^{n}\left(a_{k}\beta_{k}\phi_{k}^{\top}P_{k}\phi_{k}\right)^{2}\right) 
			=o\left(\log n\right), a.s.,
		\end{aligned}
	\end{equation}
	where we have used $(\ref{43333})$ and the fact that $$a_{k}\beta_{k}\phi_{k}^{\top}P_{k}\phi_{k}\leq 1.$$
	
	For the third term of $(\ref{s1u})$, consider the stochastic Lyapunov function $$\bar{V}_{n+1}=\tilde{\bar{\theta}}_{n+1}^{\top}\bar{P}_{n+1}^{-1}\tilde{\bar{\theta}}_{n+1},$$ where $\tilde{\bar{\theta}}_{n+1}$ is defined by $\theta-\bar{\theta}_{n+1}.$ From a similar analysis as in $(\ref{61})$, we have 
	\begin{equation}\label{suu}
		\begin{aligned}
			\bar{V}_{n+1}\leq & \bar{V}_{0}-\sum_{k=0}^{n}\left(2\bar{\beta}_{k}\tilde{\bar{\theta}}_{k}^{\top}\phi_{k}\bar{\psi}_{k}-\bar{\beta}_{k}^{2}(\tilde{\bar{\theta}}_{k}^{\top}\phi_{k})^{2}-\bar{a}_{k}\bar{\beta}_{k}^{2}\phi_{k}^{\top}\bar{P}_{k}\phi_{k}\bar{\psi}_{k}^{2}\right)\\
			&+\sum_{k=0}^{n}\bar{a}_{k}\bar{\beta}_{k}^{2}\phi_{k}^{\top}\bar{P}_{k}\phi_{k}\mathbb{E}[w_{k+1}^{2}\mid \mathcal{F}_{k}]
			-2\sum_{k=0}^{n}\left(\phi_{k}^{\top}\tilde{\bar{\theta}}_{k}-\bar{a}_{k}\bar{\psi}_{k}\phi_{k}^{\top}\bar{P}_{k}\phi_{k}\right)w_{k+1}\\
			&+\sum_{k=0}^{n}\bar{a}_{k}\phi_{k}^{\top}\bar{P}_{k}\phi_{k}\left(w_{k+1}^{2}-\mathbb{E}[w_{k+1}^{2}\mid \mathcal{F}_{k}]\right).
		\end{aligned}
	\end{equation} 
	By the definition of $\bar{\beta}_{k}$ in $(\ref{be1})$ and $\bar{\psi}_{k}$ in $(\ref{6.23})$, we have 
	\begin{equation}\nonumber
		\bar{\psi}_{k}^{2}\geq\underline{g}_{k}^{2}\left(\phi_{k}^{\top}\tilde{\bar{\theta}}_{k}\right)^{2}\geq \bar{\beta}_{k}^{2}\left(\phi_{k}^{\top}\tilde{\bar{\theta}}_{k}\right)^{2},
	\end{equation}
	and 
		\begin{eqnarray}\label{com}
			\sum_{k=0}^{n}\left(2\bar{\beta}_{k}\tilde{\bar{\theta}}_{k}^{\top}\phi_{k}\bar{\psi}_{k}-\bar{\beta}_{k}^{2}(\tilde{\bar{\theta}}_{k}^{\top}\phi_{k})^{2}-\bar{a}_{k}\bar{\beta}_{k}^{2}\phi_{k}^{\top}\bar{P}_{k}\phi_{k}\bar{\psi}_{k}^{2}\right)
			\geq  \frac{1}{2}\bar{a}_{k}\bar{\beta}_{k}^{2}\left(\phi_{k}^{\top}\tilde{\bar{\theta}}_{k}\right)^{2}.
		\end{eqnarray}
	By $(\ref{com})$ and following the similar analysis for the noise term of $(\ref{s1u})$ as in $(\ref{43333})$-$(\ref{433333})$, we can obtain
	\begin{equation}\label{4777}
		\bar{V}_{n+1}+\sum_{k=0}^{n}\frac{1}{2}\bar{a}_{k}\bar{\beta}_{k}^{2}\left(\phi_{k}^{\top}\tilde{\bar{\theta}}_{k}\right)^{2}=O\left(\log n\right),\;\;a.s.
	\end{equation}
	By $(\ref{4777})$ and the fact that $\{\bar{a}_{k},k\geq 0\}$ and $\{\bar{\beta}_{k},k\geq 0\}$ have the positive lower bound, $\{\overline{g}_{k},k\geq 0\}$ has the upper bound,  we thus obtain
	\begin{equation}\label{4999}
		\sum_{k=0}^{n}\beta_{k}^{-1}(\bar{\psi}_{k}-\beta_{k}\phi_{k}^{\top}\tilde{\bar{\theta}}_{k})^{2}=O\left(\sum_{k=0}^{n}\overline{g}_{k}^{2}\left(\phi_{k}^{\top}\tilde{\bar{\theta}}_{k}\right)^{2}\right)=O\left(\log n\right),\;\;a.s.
	\end{equation}	
	Finally, by $(\ref{s1u})$-$(\ref{433333})$ and $(\ref{4999})$, we obtain the result of Lemma \ref{lem14}.

We now give the proof of Lemma \ref{lem51}, Lemma \ref{lem3.1}, and Lemma \ref{lem6}.
	
{\bf Proof of Lemma \ref{lem51}}: 
	By $(\ref{344})$ and $(\ref{6.245})$, we have 
	\begin{equation}\label{377}
		\begin{aligned}
			&\mathbb{E}[|w_{k+1}|^{r}\mid \mathcal{F}_{k}]\\\leq&\left[\frac{f(l_{k}-\phi_{k}^{\top}\bar{\theta}_{k})}{F(l_{k}-\phi_{k}^{\top}\bar{\theta}_{k})}\right]^{r}\mathbb{E}_{k}\left[\left|\delta_{k}-F(l_{k}-\phi_{k}^{\top}\theta)\right|^{r}\right]\\
			&+\left[\frac{f(u_{k}-\phi_{k}^{\top}\bar{\theta}_{k})}{1-F(u_{k}-\phi_{k}^{\top}\bar{\theta}_{k})}\right]^{r}\mathbb{E}_{k}\left[\left|\bar{\delta}_{k}-1+F(u_{k}-\phi_{k}^{\top}\theta)\right|^{r}\right]\\
			&+\frac{1}{\sigma^{2}}\mathbb{E}_{k}\left|v_{k+1}+\phi_{k}^{\top}\tilde{\bar{\theta}}_{k}\right|^{r}.
		\end{aligned}
	\end{equation}
	From Assumption $\ref{assum2}$ and $(\ref{mm})$, we have $|\phi_{k}^{\top}\tilde{\bar{\theta}}_{k}|\leq DM$ for every $k\geq 0$. Combine with Assumption $\ref{assum4}$, there exists a positive constant $M_{r,1}$ such that $$\sup\limits_{k\geq 0}\left\{\frac{1}{\sigma^{2}}\mathbb{E}\left[\left|v_{k+1}+\phi_{k}^{\top}\tilde{\bar{\theta}}_{k}\right|^{r}\mid \mathcal{F}_{k}\right]\right\}< M_{r,1}.$$ Moreover, we have
	\begin{equation}\label{nio}
		\begin{aligned}
			&\left[\frac{f(l_{k}-\phi_{k}^{\top}\bar{\theta}_{k})}{F(l_{k}-\phi_{k}^{\top}\bar{\theta}_{k})}\right]^{r}\mathbb{E}\left[\left|\delta_{k}-F(l_{k}-\phi_{k}^{\top}\theta)\right|^{r}\mid \mathcal{F}_{k}\right]\\
			&=\left[\frac{f(l_{k}-\phi_{k}^{\top}\bar{\theta}_{k})}{F(l_{k}-\phi_{k}^{\top}\bar{\theta}_{k})}\right]^{r}\left[\left(1-F(l_{k}-\phi_{k}^{\top}\theta)\right)^{r}F(l_{k}-\phi_{k}^{\top}\theta)\right]\\
			&+\left[\frac{f(l_{k}-\phi_{k}^{\top}\bar{\theta}_{k})}{F(l_{k}-\phi_{k}^{\top}\bar{\theta}_{k})}\right]^{r}\left[\left(F(l_{k}-\phi_{k}^{\top}\theta)\right)^{r}\left(1-F(l_{k}-\phi_{k}^{\top}\theta)\right)\right].
		\end{aligned}
	\end{equation}
	Since $|\phi_{k}^{\top}\theta-\phi_{k}^{\top}\bar{\theta}_{k}|\leq 2DM$, we have
	\begin{equation}
		\begin{aligned}
			&\left|\left[\frac{f(l_{k}-\phi_{k}^{\top}\bar{\theta}_{k})}{F(l_{k}-\phi_{k}^{\top}\bar{\theta}_{k})}\right]^{r}\left[\left(1-F(l_{k}-\phi_{k}^{\top}\theta)\right)^{r}F(l_{k}-\phi_{k}^{\top}\theta)\right]\right|\\
			\leq&\left[\frac{f(l_{k}-\phi_{k}^{\top}\bar{\theta}_{k})}{F(l_{k}-\phi_{k}^{\top}\bar{\theta}_{k})}\right]^{r}F(l_{k}-\phi_{k}^{\top}\bar{\theta}_{k}+2DM).
		\end{aligned}
	\end{equation}
	Now, define the function %
		$g(x)\triangleq \left[\frac{f(x)}{F(x)}\right]^{r}F(x+2DM).$
	Let $x\rightarrow -\infty$, $g(x)$ tends to $0$ by L'H\^opital's Rule. Besides, by Assumption \ref{assum2}, let the constant $C$ denote as follows:
	\begin{equation}\label{sa}
		C=\sup_{k\geq 0}\{|l_{k}|+|u_{k}|\}.
	\end{equation}
	Thus, the function $g(\cdot)$ is bounded on $[-\infty, C+DM]$. Since $l_{k}-\phi_{k}^{\top}\bar{\theta}_{k}\in [-\infty, C+2DM]$ for every $k\geq 0$, we thus obtain that there exists a positive constant $M_{r,2}$ such that
	$$\sup_{k\geq 0}\left\{\frac{f^{r}(l_{k}-\phi_{k}^{\top}\bar{\theta}_{k})}{F^{r}(l_{k}-\phi_{k}^{\top}\bar{\theta}_{k})}\mathbb{E}\left[\left|\delta_{k}-F(l_{k}-\phi_{k}^{\top}\theta)\right|^{r}\mid \mathcal{F}_{k}\right]\right\}< M_{r,2}.$$ 
	Similarly, there will be a constant $M_{r,3}$ such that 
	\begin{equation}
	\begin{aligned}
	&\sup_{k\geq 0}\left\{\left[\frac{f(u_{k}-\phi_{k}^{\top}\bar{\theta}_{k})}{1-F(u_{k}-\phi_{k}^{\top}\bar{\theta}_{k})}\right]^{r}\mathbb{E}_{k}\left|\delta_{k}-1+F(u_{k}-\phi_{k}^{\top}\theta)\right|^{r}\right\}\\
	&< M_{r,3}.
	\end{aligned}
	\end{equation}
	Thus, we obtain $\mathbb{E}\left[|w_{k+1}|^{r}\mid \mathcal{F}_{k}\right]<M_{r,1}+M_{r,2}+M_{r,3},$ which deduce Lemma $\ref{lem51}$. 

{\bf Proof of Lemma  \ref{lem3.1}}:
	By the definition of the function $G_{k}(x,y)$ in $(\ref{3.6})$, we have 
	\begin{equation}\label{g}
		\begin{aligned}
			\frac{\partial G_{k}(y,x)}{\partial x}
			=&\left[\frac{f(l_{k}-y)}{F(l_{k}-y)}-\frac{y-l_{k}}{\sigma^{2}}\right]f(l_{k}-x)\\
			&+\left[\frac{f(u_{k}-y)}{1-F(u_{k}-y)}-\frac{u_{k}-y}{\sigma^{2}}\right]f(u_{k}-x)
			+\frac{1}{\sigma^{2}}[F(u_{k}-x)-F(l_{k}-x)].
		\end{aligned}
	\end{equation}	
	We now analyze the RHS of the equation $(\ref{g})$ case by case. If $l_{k}<-2C-2\Gamma$, by $(\ref{sa})$, we have for every $|x|\leq \Gamma, |y|\leq \Gamma$,
	\begin{equation}
		\frac{1}{\sigma^{2}}[F(u_{k}-x)-F(l_{k}-x)]\geq \frac{1}{\sigma^{2}}[F(-C-\Gamma)-F(-2C-\Gamma)].
	\end{equation}  
	Besides, let $H(x,y)=f(x)f(y)-F(x)f'(y)$, we have
	\begin{equation}\label{3.10}
		\frac{\partial H(x,y)}{\partial x}=f'(x)f(y)-f(x)f'(y)=\frac{y-x}{\sigma^{2}}f(x)f(y).
	\end{equation}
	From $(\ref{3.10})$, we can easily obtain that $\frac{\partial H(x,y)}{\partial x}\geq 0$ when $x\leq y$. Therefore, for every $x\leq y$, since $\lim
	\limits_{x\rightarrow -\infty}H(x,y)=0,$ we have $H(x,y)\geq 0$. Thus, we obtain that
	\begin{equation}\label{3.11}
		\frac{-f(l_{k}-y)}{F(l_{k}-y)}\leq\frac{-f'(l_{k}-y)}{f(l_{k}-y)}=\frac{l_{k}-y}{\sigma^{2}},
	\end{equation}
	where we have used the fact that $f'(x)=-\frac{x}{\sigma^{2}}f(x)$ for the density function of the normal distribution. Similarly, we will have
	\begin{equation}\label{3.12}
		\frac{f(u_{k}-y)}{1-F(u_{k}-y)} \geq \frac{-f'(u_{k}-y)}{f(u_{k}-y)}=\frac{u_{k}-y}{\sigma^{2}}.
	\end{equation}
	Therefore, by $(\ref{g})$, $(\ref{3.11})$ and $(\ref{3.12})$, we have 
	\begin{equation}\label{1234}
		\frac{\partial G_{k}(y,x)}{\partial x}\geq \frac{1}{\sigma^{2}}[F(-C-\Gamma)-F(-2C-\Gamma)]>0.
	\end{equation}
	Similarly, if $u_{k}>2C+2\Gamma$, $(\ref{1234})$ will also be true. 
	
	For the case $-2C-2\Gamma<l_{k}\leq C$ and $-C\leq u_{k}<2C+2\Gamma$, we have $\frac{1}{\sigma^{2}}[F(u_{k}-x)-F(l_{k}-x)]\geq \frac{u_{k}-l_{k}}{\sigma^{2}}f(\xi_{k}-x)>0,$	
	where $\xi_{k}\in[l_{k}, u_{k}]$. Thus, we have
	\begin{equation}\label{a10}
		\begin{aligned}
			\frac{\partial G_{k}(y,x)}{\partial x}\geq  2\left[f\left(2C+3\Gamma\right)\right]^{2}>0.
		\end{aligned}
	\end{equation}
From $(\ref{1234})$ and $(\ref{a10})$, we finally obtain $(\ref{iinf})$.	
	
	For the proof of $(\ref{355})$, we have 
	\begin{equation}\label{1366}
		\begin{aligned}
			\frac{\partial^{2} G_{k}(y,x)}{\partial x^{2}}
			=&\left[\frac{f(l_{k}-y)}{F(l_{k}-y)}-\frac{y-l_{k}}{\sigma^{2}}\right]\frac{x-l_{k}}{\sigma^{2}}f(l_{k}-x)\\
			&+\left[\frac{f(u_{k}-y)}{1-F(u_{k}-y)}-\frac{u_{k}-y}{\sigma^{2}}\right]\frac{x-u_{k}}{\sigma^{2}}f(u_{k}-x)\\
			&+\frac{1}{\sigma^{2}}\left[f(l_{k}-x)-f(u_{k}-x)\right]
		\end{aligned}
	\end{equation}
	Notice that for every $|x|, |y|\leq \Gamma$, the first term of the LHS of $(\ref{1366})$ tends to 0 as $l_{k} \rightarrow -\infty$, and the second term of the LHS of $(\ref{1366})$ tends to 0 as $u_{k} \rightarrow \infty$. Therefore, we can obtain $(\ref{355})$.
	The proof of the equation $(\ref{a43})$ and  $(\ref{a44})$ can be obtained through a similar analysis.

{\bf Proof of Lemma \ref{lem6}}: 
	From $(\ref{suu})$, $(\ref{com})$, and the fact that $\bar{a}_{k}$ and $\bar{\beta}_{k}$ have positive lower bounds, we have 
	\begin{equation}\label{sq}
	\setlength\abovedisplayskip{6pt}
\setlength\belowdisplayskip{6pt}
		\begin{aligned}
			&\mathbb{E}[\bar{V}_{n+1}^{r}]+\mathbb{E}\left[\sum_{k=0}^{n}\left|\phi_{k}^{\top}\tilde{\bar{\theta}}_{k}\right|^{2}\right]^{r}\\
			=&O\left(\mathbb{E}\left[\sum_{k=0}^{n}\left(\bar{\beta}_{k}\phi_{k}^{\top}\tilde{\bar{\theta}}_{k}-\bar{a}_{k}\bar{\beta}_{k}^{2}\bar{\psi}_{k}\phi_{k}^{\top}\bar{P}_{k}\phi_{k}\right)w_{k+1}\right]^{r}\right)\\
			&+O\left(\mathbb{E}\left[\sum_{k=0}^{n}\bar{a}_{k}\bar{\beta}_{k}^{2}\phi_{k}^{\top}\bar{P}_{k}\phi_{k}w_{k+1}^{2}\right]^{r}\right).	
		\end{aligned}
	\end{equation}
	 For each $k\geq 0$ and $t\geq 0$, we denote
	\begin{equation}\label{5.13}
		\setlength\abovedisplayskip{6pt}
\setlength\belowdisplayskip{6pt}
		\left\{
		\begin{array}{rcl}
			W_{k+1}^{(t)}&=&\left(W_{k+1}^{(t-1)}\right)^{2}-\mathbb{E}\left[\left(W_{k+1}^{(t-1)}\right)^{2}\mid \mathcal{F}_{k}\right],\;\;t\geq 1,\\
			W_{k+1}^{(0)}&=&w_{k+1}.
		\end{array}
		\right.
	\end{equation}
	We now first prove the following property by $(\ref{sq})$:
	\begin{equation}\label{a20}
		\setlength\abovedisplayskip{6pt}
\setlength\belowdisplayskip{6pt}
	\mathbb{E}\left[\bar{V}_{n+1}^{r}\right]+\mathbb{E}\left[\sum\limits_{k=0}^{n}\left(\phi_{k}^{\top}\tilde{\bar{\theta}}_{k}\right)^{2}\right]^{r}= O(\log^{r} n).
	\end{equation}

For the case $r=1$, 	since $\sup_{k\geq 0}\mathbb{E}\left[w_{k+1}^{2}\mid \mathcal{F}_{k}\right]$ is bounded, by Lemma \ref{lem3}, we have 
\begin{equation}\label{6.9}
	\setlength\abovedisplayskip{6pt}
\setlength\belowdisplayskip{6pt}
	\begin{aligned}
		\sum_{k=0}^{n}\mathbb{E}\left[\bar{a}_{k}\bar{\beta}_{k}^{2}\phi_{k}^{\top}\bar{P}_{k}\phi_{k}\mathbb{E}\left[w_{k+1}^{2}\mid \mathcal{F}_{k}\right]\right]
		=O\left(\mathbb{E}\left[\log \left|\bar{P}_{n+1}^{-1}\right|\right]\right)=O\left(\log n\right).
		\end{aligned}
	\end{equation}
	By $(\ref{sq})$, $(\ref{6.9})$, and the fact that $\{w_{k},\mathcal{F}_{k}\}$ is a martingale deference sequence, we obtain $(\ref{a20})$ is true for the case $r=1$. If $(\ref{a20})$ holds for each $s\leq r_{0}-1, r_{0}\geq 2$, then by Lemma \ref{lem3} and Lemma \ref{lem1}, we will have
	\begin{equation}\label{5.19}
		\setlength\abovedisplayskip{6pt}
\setlength\belowdisplayskip{6pt}
		\begin{aligned}
			&\mathbb{E}\left[\sum_{k=0}^{n}\bar{a}_{k}\bar{\beta}_{k}^{2}\phi_{k}^{\top}\bar{P}_{k}\phi_{k}w_{k+1}^{2}\right]^{r_{0}}\\
			=&O\left(\mathbb{E}\left[\sum_{k=0}^{n}\bar{a}_{k}\bar{\beta}_{k}^{2}\phi_{k}^{\top}\bar{P}_{k}\phi_{k}\mathbb{E}\left(w_{k+1}^{2}\mid \mathcal{F}_{k}\right)\right]^{r_{0}}\right)+O\left(\mathbb{E}\left[\sum_{k=0}^{n}\bar{a}_{k}\bar{\beta}_{k}^{2}\phi_{k}^{\top}\bar{P}_{k}\phi_{k}\left(W_{k+1}^{(1)}\right)^{2}\right]^{\frac{r_{0}}{2}}\right)\\
			=&O(\log^{r_{0}} n)+O\left(\mathbb{E}\left[\sum_{k=0}^{n}\bar{a}_{k}\bar{\beta}_{k}^{2}\phi_{k}^{\top}\bar{P}_{k}\phi_{k}\left(W_{k+1}^{(1)}\right)^{2}\right]^{\frac{r_{0}}{2}}\right),
		\end{aligned}
	\end{equation} 
	where we have used the fact that $\bar{a}_{k}\bar{\beta}_{k}^{2}\phi_{k}^{\top}\bar{P}_{k}\phi_{k}<1$. Besides, by (\ref{com}), Lemma \ref{lem1}  and given that (\ref{a20}) is valid for $r=\frac{r_{0}}{2}$, we obtain
	\begin{equation}\label{a24}
		\begin{aligned}
			&\mathbb{E}\left[\left(\sum_{k=0}^{n}\left(\bar{\beta}_{k}\phi_{k}^{\top}\tilde{\bar{\theta}}_{k}-\bar{a}_{k}\bar{\beta}_{k}^{2}\bar{\psi}_{k}\phi_{k}^{\top}\bar{P}_{k}\phi_{k}\right)w_{k+1}\right)^{r_{0}}\right]\\
			=&O\left(\mathbb{E}\left[\sum_{k=0}^{n}\left|\phi_{k}^{\top}\tilde{\bar{\theta}}_{k}\right|^{2}\mathbb{E}\left(w_{k+1}^{2}\mid \mathcal{F}_{k}\right)\right]^{\frac{r_{0}}{2}}\right)
			+O\left(\mathbb{E}\left[\sum_{k=0}^{n}\left|\phi_{k}^{\top}\tilde{\bar{\theta}}_{k}\right|^{2}W_{k+1}^{(1)}\right]^{\frac{r_{0}}{2}}\right)\\
			=&O\left(\log^{\frac{r_{0}}{2}} n\right)+O\left(\mathbb{E}\left[\sum_{k=0}^{n}\left|\phi_{k}^{\top}\tilde{\bar{\theta}}_{k}\right|^{2}W_{k+1}^{(1)}\right]^{\frac{r_{0}}{2}}\right).
		\end{aligned}
	\end{equation}
Combine $(\ref{sq})$, $(\ref{5.19})$ and $(\ref{a24})$, we have
	\begin{equation}\label{a21}
	\begin{aligned}
		&\mathbb{E}\left[\bar{V}_{n+1}^{r_{0}}\right]+\mathbb{E}\left[\left(\sum_{k=0}^{n}\left|\phi_{k}^{\top}\tilde{\bar{\theta}}_{k}\right|^{2}\right)^{r_{0}}\right]\\
		=&O(\log^{r_{0}} n)+O\left(\mathbb{E}\left[\sum_{k=0}^{n}\left|\phi_{k}^{\top}\tilde{\bar{\theta}}_{k}\right|^{2}W_{k+1}^{(1)}\right]^{\frac{r_{0}}{2}}\right)\\
		&+O\left(\mathbb{E}\left[\sum_{k=0}^{n}\bar{a}_{k}\bar{\beta}_{k}^{2}\phi_{k}^{\top}\bar{P}_{k}\phi_{k}(W_{k+1}^{(1)})^{2}\right]^{\frac{r_{0}}{2}}\right).
		\end{aligned}
	\end{equation} 	
To analyze the last two noise terms of (\ref{a21}), we can continue the recursive process according to $(\ref{5.19})$-$(\ref{a24})$, which allows us to deduce that (\ref{a20}) holds when $r=r_{0}.$ Thus, by induction, (\ref{a20}) holds.

Furthermore, by $(\ref{s1u})$, we have 
\begin{equation}\label{sqq}
		\begin{aligned}
			&\mathbb{E}[V_{n+1}^{r}]+\mathbb{E}\left[\sum_{k=0}^{n}a_{k}\psi_{k}^{2}\right]^{r}\\
			=&O\left(\mathbb{E}\left[\sum_{k=0}^{n}(\bar{\psi}_{k}-\beta_{k}\phi_{k}^{\top}\tilde{\bar{\theta}}_{k})^{2}\right]^{r}\right)+O\left(\mathbb{E}\left[\sum_{k=0}^{n}2(\bar{\psi}_{k}-\beta_{k}\phi_{k}^{\top}\tilde{\bar{\theta}}_{k})w_{k+1}\right]^{r}\right)\\
		&+O\left(\mathbb{E}\left[\sum_{k=0}^{n}a_{k}\psi_{k}w_{k+1}\right]^{r}\right)
			+O\left(\mathbb{E}\left[\sum_{k=0}^{n}a_{k}\beta_{k}\phi_{k}^{\top}P_{k}\phi_{k}w_{k+1}^{2}\right]^{r}\right).	
		\end{aligned}
	\end{equation}
By $(\ref{a20})$ and the fact that $\left\{\overline{g}_{k}, k\geq 0\right\}$ have upper bound, we have
\begin{equation}
\begin{aligned}
&\mathbb{E}\left[\sum_{k=0}^{n}(\bar{\psi}_{k}-\beta_{k}\phi_{k}^{\top}\tilde{\bar{\theta}}_{k})^{2}\right]^{r}=O\left(\mathbb{E}\left[\sum_{k=0}^{n}\overline{g}_{k}^{2}(\phi_{k}^{\top}\tilde{\bar{\theta}}_{k})^{2}\right]^{r}\right)\\
=&O\left(\mathbb{E}\left[\sum_{k=0}^{n}\left|\phi_{k}^{\top}\tilde{\bar{\theta}}_{k}\right|^{2}\right]^{r}\right)=O\left(\log^{r}n\right).
\end{aligned}
\end{equation}	
The asymptotic analysis of the last three noise terms in (\ref{sqq}) can be conducted similarly to the analysis of the noise terms on the right side of the equation (\ref{sq}). Finally, we can obtain:
\begin{equation}\label{5.26}
		\begin{aligned}
			\mathbb{E}\left[V_{n+1}^{r}\right]+\mathbb{E}\left[\left(\sum_{k=0}^{n}a_{k}\psi_{k}^{2}\right)^{r}\right]=& O(\log^{r} n),
		\end{aligned}
	\end{equation}
	which completes the proof of the lemma.

\end{appendix}	
	
\bibliographystyle{siamplain}
\bibliography{efficiency}

\end{document}